\documentclass[10pt,twocolumn,letterpaper]{article}

\usepackage{cvpr}
\usepackage{times}
\usepackage{epsfig}
\usepackage{graphicx}
\usepackage{amsmath}
\usepackage{amsthm}
\usepackage{amssymb}
\usepackage{xspace}
\usepackage{paralist}
\usepackage{subcaption}
\usepackage{bbold}

\usepackage{tabularx}
\usepackage{booktabs}
\usepackage[numbers,sort]{natbib}

\usepackage[hang,flushmargin]{footmisc}

\usepackage{algorithm}
\usepackage{algorithmicx}
\usepackage[noend]{algpseudocode}
\algnewcommand{\LineComment}[1]{\Statex \(\triangleright\) #1}

\usepackage{adjustbox}

\newcommand{\set}[1]{\ensuremath{\mathcal{#1}}}
\newcommand{\con}[1]{#1} %\ensuremath{\mathsf{#1}}}

\newcommand{\argmax}{\operatornamewithlimits{\arg\,\max}}
\newcommand{\argmin}{\operatornamewithlimits{\arg\,\min}}

\newcommand{\baseOnly}{\texttt{BASE}\xspace }
\newcommand{\base}{\texttt{BASE$_+$}\xspace }

\newcommand{\lmp}{\texttt{LMP}\xspace }

\newcommand{\myparagraph}[1]{\vspace{5pt}\noindent\textbf{#1}}

\newtheorem{theorem}{Theorem}

\newtheorem{proposition}{Proposition}

\usepackage[hyphens]{url}

\usepackage[pagebackref=true,breaklinks=true,letterpaper=true,colorlinks,bookmarks=false]{hyperref}

\cvprfinalcopy % *** Uncomment this line for the final submission

\def\cvprPaperID{778} % *** Enter the CVPR Paper ID here

\begin{document}

%%%%%%%%% TITLE
\title{Loss Max-Pooling for Semantic Image Segmentation}

\author{
    Samuel Rota Bul\`o$^{\star,\dagger}$\hspace{.15\columnwidth} Gerhard Neuhold$^\dagger$\hspace{.15\columnwidth} Peter Kontschieder$^\dagger$\\[5pt]
    $^\star$FBK - Trento, Italy - {\tt\small rotabulo@fbk.eu}\\
    $^\dagger$Mapillary - Graz, Austria - {\tt\small \{samuel,gerhard,pkontschieder\}@mapillary.com}
}
%    \parbox{\columnwidth}{
%        \centering
%        $^\star$FBK\\
%        Trento, Italy\\
%        {\tt\small rotabulo@fbk.eu}
%    }
%    \hspace{.25\columnwidth}
%    \parbox{\columnwidth}{
%        \centering
%$^\dagger$Mapillary\\
%Graz, Austria\\
%{\tt\small \{samuel,gerhard,pkontschieder\}@mapillary.com}
%    }
%FBK\\
%Trento, Italy\\
%{\tt\small rotabulo@fbk.eu}
%\and
%Gerhard Neuhold\\
%Mapillary\\
%Graz, Austria\\
%{\tt\small gerhard@mapillary.com}
%\and
%Peter Kontschieder\\
%Mapillary\\
%Graz, Austria\\
%{\tt\small pkontschieder@mapillary.com}

\maketitle
%\thispagestyle{empty}

%%%%%%%%% ABSTRACT

\begin{abstract}
   \vspace{-10pt}
   We introduce a novel loss max-pooling concept for handling imbalanced training data distributions, applicable as alternative loss layer in the context of deep neural networks for semantic image segmentation. Most real-world semantic segmentation datasets exhibit long tail distributions with few object categories comprising the majority of data and consequently biasing the classifiers towards them. Our method adaptively re-weights the contributions of each pixel based on their observed losses, targeting under-performing classification results as often encountered for under-represented object classes. Our approach goes beyond conventional cost-sensitive learning attempts through adaptive considerations that allow us to indirectly address both, inter- and intra-class imbalances. We provide a theoretical justification of our approach, complementary to experimental analyses on benchmark %standard semantic segmentation 
datasets. In our experiments on the Cityscapes and Pascal VOC 2012 segmentation datasets we find consistently improved results, demonstrating the efficacy of our approach.
\end{abstract}

%%%%%%%%% BODY TEXT
\vspace{-15pt}
\section{Introduction}
Deep learning approaches have undoubtedly matured to the new de facto standards for many traditional computer vision tasks like image classification, object detection or semantic segmentation. Semantic segmentation aims to assign categorical labels to each pixel in an image and therefore constitutes the basis for high-level image understanding. Recent works have contributed to the progress in this research field by building upon convolutional neural networks (CNNs)~\cite{Lecun98b} and enriching them with task-specific functionalities. Extending CNNs to directly cast dense, semantic label maps~\cite{Long2015,Badrinarayanan2015}, including more contextual information~\cite{Liu2015,Chen2016,Yu2016,Ghiasi2016} or refining results with graphical models~\cite{Lin2015,Zheng2015}, have led to impressive results in many real-world applications and on standard benchmark datasets.
%or many of these applications, specifically designed layers were key to solving task-specific problems. For example, state-of-the-art object detection methods benefit from RoI-pooling layers (providing max-responses on canonical box views) or semantic segmentation results are vastly improved when using convolutional layers with dilated kernels allowing to better capture contextual information. \peter{put refs here and in next paragraph}

Few works have focused on how to properly handle imbalanced (or \textit{skewed}) class distributions, as often encountered in semantic segmentation datasets, within deep neural network training so far. With imbalanced, we refer to datasets having dominant portions of their data assigned to (few) majority classes while the rest belongs to minority classes, forming comparably under-represented categories. As (mostly undesired) consequence, it can be observed that classifiers trained without correction mechanisms tend to be biased towards the majority classes during inference. 

One way to mitigate this \textit{class-imbalance} problem is to emphasize on balanced compilations of datasets in the first place by collecting their samples approximately uniformly. Datasets following such an approach are ImageNet~\cite{Deng2009}, Caltech 101/256~\cite{FeiFei2006,Griffin2007} or CIFAR 10/100~\cite{CIFAR10}, where training, validation and test sets are roughly balanced \wrt~the instances per class. Another widely used procedure is conducting over-sampling of minority classes or under-sampling from the majority classes when compiling the actual training data. Such approaches are known to change the underlying data distributions and may result in suboptimal exploitation of available data, increased computational effort and/or risk of over-fitting when repeatedly visiting the same samples from minority classes (\cf~SMOTE and derived variants~\cite{Chawla2002,Han2005,Bunkhumpornpat2009,Jeatrakul2010} on ways to avoid over-fitting). However, its efficiency and straightforward application for tasks like image-level classification rendered sampling a commonly-agreed practice. 

Another approach termed \textit{cost-sensitive} learning changes the algorithmic behavior by introducing class-specific weights, often derived from the original data statistics. Such methods were recently investigated~\cite{Xu2014,Xu2015,Mostajabi2015,Caesar2015} for deep learning, some of them following ideas previously applied in shallow learning methods like random forests~\cite{Khoshgoftaar2007,Kontschieder2013} or support vector machines~\cite{Raskutti2004,Tang2009}. Many of these works use statically-defined cost matrices~\cite{Xu2014,Eigen2015,Xu2015,Mostajabi2015,Caesar2015} or introduce additional parameter learning steps~\cite{Khan2015}. Due to the spatial arrangement and strong correlations of classes between adjacent pixels, cost-sensitive learning techniques are preferred over resampling methods when performing dense, pixel-wise classification as in semantic segmentation tasks. However, current trends of semantic segmentation datasets show strong increase in complexity with more minority classes being added. %that are covering very small fractions of the overall pixel mass.

%(sometimes also referred to as stuff or background classes) while the sum of most object-specific (minority) categories accounts for only few percent of the labelled pixel mass per image.

%Indeed, when dealing with approximately uniformly compiled datasets like ImageNet~\cite{}, special efforts for imbalance correction can be mitigated. , due to careful compilation of the dataset. despite being commonly present in most computer vision datasets.  Such sampling techniques are operating on a data level (\ie are changing the original distribution) and are typically not applicable when dealing with semantic segmentation tasks.  Consequently, compiling training batches with data equally distributed over all classes to be learnt is impossible due to the spatial layout and strong correlation of classes between neighboring pixels. 

\myparagraph{Contributions.}
In this work we propose a principled solution to handling imbalanced datasets within deep learning approaches for semantic segmentation tasks. Specifically, we introduce a novel loss function, which upper bounds the traditional losses where the contribution of each pixel is weighted equally.
The upper bound is obtained via a generalized max-pooling operator acting at the pixel-loss level. The maximization is taken with respect to pixel weighting functions, thus providing an adaptive re-weighting of the contributions of each pixel, based on the loss they actually exhibit. In general, pixels incurring higher losses during training are weighted more than pixels with a lower loss, thus indirectly compensating potential inter-class and intra-class imbalances within the dataset. The latter imbalance is approached because our dynamic re-weighting is class-agnostic,~\ie~we are not taking advantage of the class label statistics like previous cost-sensitive learning approaches. %, and can therefore even cope with potential \textit{intra}-class imbalances. 

The generalized max-pooling operator, and hence our new loss, can be instantiated in different ways depending on how we delimit the space of feasible pixel weighting functions. In this paper, we focus on a particular family of weighting functions with bounded $p$-norm and $\infty$-norm, and study the properties that our loss function exhibits under this setting. Moreover, we provide the theoretical contribution of deriving an explicit characterization of our loss function under this special case, which enables the computation of gradients that are needed for the optimization of the deep neural network. 

As additional, complementary contribution we describe a performance-dependent sampling approach, guiding the minibatch compilation during training. By keeping track of the prediction performance on the training set, we show how a relatively simple change in the sampling scheme allows us to faster reach convergence and improved results. 

The rest of this section discusses some related works and how current semantic segmentation approaches typically deal with the class-imbalance problem, before we provide a compact description for the notation used in the rest of this paper. In Sect.~\ref{sec:method} we describe how we depart from the standard, uniform weighting scheme to our proposed adaptive, pixel-loss max-pooling and the space of weighting functions we are considering. Sect.~\ref{sec:Computation} and~\ref{sec:alg} describe how we eventually solve the novel loss function and provide algorithmic details, respectively. In Sect.~\ref{sec:Experiments} we assess the performance of our contributions on the challenging Cityscapes and Pascal VOC segmentation benchmarks before we conclude in Sect.~\ref{sec:Conclusion}.
Note that this is an extended version of our CVPR 2017 paper~\cite{RotNeuKon17cvpr}. %Finally, we want to refer to the supplementary material where we provide in-depth analyses and correctness proofs for our approach.

%To conclude, we remark that our approach is class-agnostic,~\ie we are not taking advantage of the class label statistics like previous cost-sensitive learning approaches, and can therefore even cope with potential \textit{intra}-class imbalances. 

%Our method upper bounds the conventionally used, pixel-wise log-loss and adaptively applies a re-weighting scheme, acting based on the actually observed pixel losses. We focus on weighting functions with \textit{p}-norm and additionally allow to parameterize a minimum number of supported pixels by the optimal weighting function. 
%As a result, our algorithm identifies samples at pixel-level that are underperforming, given the current training state of the CNN. 
%To conclude, we remark that our approach is class-agnostic,~\ie we are not taking advantage of the class label statistics like previous cost-sensitive learning approaches, and can therefore even cope with potential \textit{intra}-class imbalances. 

%However, such sampling , which is however not applicable for semantic labelling tasks due to the semantic composition of objects. Moreover, sampling approaches often either inefficiently exploit the majority class while showing tendencies to 

\myparagraph{Related Works.}
Many semantic segmentation works follow a relatively simple cost-sensitive approach via an inverse frequency rebalancing scheme,~\eg~\cite{Xu2014,Xu2015,Mostajabi2015,Caesar2015} or median frequency re-weighting~\cite{Eigen2015}. Other approaches construct best-practice heuristics by~\eg restricting the number of pixels to be updated during backpropagation: The work in~\cite{Bansal2016} suggests increasing the minibatch size while decreasing the absolute number of (randomly sampled) pixel positions to be updated. In~\cite{Wu2016}, an approach coined online bootstrapping is introduced, where pixel losses are sorted and only the \textit{k} highest loss positions are updated. A similar idea termed \textit{online hard example mining}~\cite{Shrivastava2016} was found to be effective for object detection, where high-loss bounding boxes retained after a non-maximum-suppression step were preferably updated. The work in~\cite{Huang2016} tackles class imbalance via enforcing inter-cluster and inter-class margins, obtained by employing quintuplet instance sampling with a triple-header hinge loss. Another recent work~\cite{Khan2015} proposed a cost-sensitive neural network for classification, jointly optimizing for class dependent costs and the standard neural network parameters. The work in~\cite{Shen2015} addresses the problem of contour detection with convolutional neural networks (CNN), combining a specific loss for contour versus non-contour samples with the conventional log-loss. In separate though related research fields, focus was put on directly optimizing the target measures like Area under curve (AUC), Intersection over Union (IoU or Jaccard Index) or Average class (AC)~\cite{Blaschko2008,Ranjbar2010,Nowozin2014,Ahmed2015}. The work in~\cite{Gulcehre2013} is proposing a nonlinear activation function computing the $L_p$ norm of projections from the lower layers, allowing to interpret max-, average- and root-mean-squared-pooling operators as special cases of their activation function.

\myparagraph{Notation.}
In this paper, we denote by $\set A^{\set B}$ the space of functions mapping elements in the set $\set B$ to elements in the set $\set A$,
while $\set A^n$ with $n$ a natual number denotes the usual product set of $n$-tuples with elements in $\set A$.
%Function composition is denoted by $\circ$.
%Given $f\in\set A^{\set B}$, 
%we use notations $f_b$ in place of $f(b)$, unless differently specified.
The sets of real and integer numbers are $\mathbb R$ and $\mathbb Z$, respectively.
Let $f,g\in\mathbb R^\set A$, $c\in\mathbb R$. 
Operations defined on $\mathbb R$ such as, \eg addition, multiplication, exponentiation, \etc, are inherited by $\mathbb R^{\set A}$ via pointwise application (for instance, $f+g$ is the function $z\in\set A\mapsto f(z)+g(z)$ and $f^c$ is the function $z\in\set A\mapsto f(z)^c$).
Additionally, we use the notations:
\begin{compactitem}
%\item $\sum_{\set B} f = \sum_{z\in\set A\cap\set B}f(z)$, and $\sum f=\sum_\set A f$
\item $\langle f\rangle_{\set B}  = \sum_{z\in\set A\cap\set B}f(z)$, and $\langle f\rangle=\langle f \rangle_{\set A}$
\item $\Vert f\Vert_{p,\set B} = \left(\sum_{z\in\set A\cap\set B}f(z)^p\right)^{1/p}$ and $\Vert f\Vert_p=\Vert f\Vert_{p,\set A}$
\item $f\cdot g=\sum_{z\in\set A}f(z)g(z)$
\item $f\preceq c \iff (\forall z\in\set A)(f(z)\leq c)$
\item $(f)_+$ denotes the function $z\in\set A\mapsto \max\{f(z),0\}$.
\end{compactitem}

%Given $f\in\mathbb R^{\set A}$, we have $\sum f=\sum_{z\in\set A} f(z)$ and $\sum_{\set B} f = \sum_{z\in\set A\cap\set B}f(z)$.
%Moreover, $\Vert f\Vert_p=\left(\sum_{z\in\set A}f_z^p\right)^{1/p}$ is the $p$-norm of $f$ and $\Vert f\Vert_{p,\set B} = \left(\sum_{z\in\set A\cap\set B}f_z^p\right)^{1/p}$. 
%Given $f\in\mathbb R^{\set A}$ and a constant $c\in\mathbb R$ we have $f\preceq c \iff (\forall z\in\set A)(f_z\leq c)$. Moreover, if $\star$ is a binary operator in $\mathbb R$ then $f\star c$ is a function in $\mathbb R^\set A$ defined as $(f\star c)_z = f_z \star c$.
%Finally, $\langle f,g\rangle$

%Given $f\in\mathbb R^\set A$ and $g\in\mathbb R^\set B$, the dot product between $f$ and $g$ is defined as $\langle f,g\rangle=\sum_{z\in\set A\cap\set B}f_z g_z$ the dot product between $a$ and $b$, 

\section{Pixel-Loss Max-Pooling}\label{sec:method}
The goal of semantic image segmentation is to provide an assignment of class labels to each pixel of an image.
The input space for this task is denoted by $\set X$ and corresponds to the set of possible images.
For the sake of simplicity, we assume all images to have the same number of pixels . We denote by $\set I\subset\set Z^2$ the set of pixels within an image, and let $\con n$ be the number of pixels, \ie $\con n=|\set I|$. 
The output space for the segmentation task is denoted by $\set Y$ and corresponds to all pixelwise labelings with classes in $\set C$.
Each labeling $y\in\set Y$ is a function mapping pixels to classes, \ie $\set Y=\set C^\set I$.
%(\ie with domain $\set I$ and co-domain $\set C$). 
%Moreover,  we write $y_{\vct u}$ for $y(\vct u)$, which yields the class label of pixel $\vct u$ for labelling $y$.

\myparagraph{Standard setting.}
The typical objective used to train a model $f_\theta\in\set Y^\set X$  with parameters $\theta$ (\eg a fully-convolutional network),
given a training set $\set T\subset\set X\times\set Y$, takes the following form:
\begin{equation}
    \min\left\{ \sum_{(x,y)\in\set T} L(f_\theta(x),y) +\lambda R(\theta)\,:\,\theta\in\Theta\right\}\,,
    \label{eq:training_objective}
\end{equation}
where $\Theta$ is the set of possible network parameters, $L \in\mathbb R^{\set Y\times\set Y}$ is a loss function penalizing wrong image labelings and $R\in\mathbb R^\Theta$ is a regularizer.
The loss function $L$ commonly decomposes into a sum of pixel-specific losses as follows
\begin{equation}
    L(\hat y,y)=\frac{1}{\con n}\langle\ell_{\hat y y}\rangle\,, %\frac{1}{\con n}\sum_{\vct u\in\set I} \ell(\hat y_{\vct u},y_{\vct u})=\,,
    \label{eq:image_loss}
\end{equation}
where  $\ell_{\hat yy}\in\mathbb R^\set I$ assigns to each pixel $u\in\set I$ the loss incurred 
for predicting class $\hat y(u)$ instead of $y(u)$.
In the rest of the paper, we assume $\ell_{\hat y y}$ to be \emph{non-negative} and \emph{bounded} (\ie pixel losses are finite).

\myparagraph{Loss max-pooling.} The loss function defined in \eqref{eq:image_loss} weights uniformly the contribution of each pixel within the image. %It can be considered as the outcome of an average-pooling applied to the pixels' losses.
The effect of this choice is a bias of the learner towards elements that are dominant within the image (\eg sky, building, road) to the detriment of elements occupying smaller portions of the image. 
%\mycomment{argue that this has a negative effect on the generalization capability, even when we are able to perfectly predict training data, and maybe weighting yields faster convergence ?}
In order to alleviate this issue, we propose to %consider to apply a generalized max-pooling operator at the pixel-losses level.
adaptively reweigh the contribution of each pixel based on the actual loss we observe. 
Our goal is to shift the focus on image parts where the loss is higher, while retaining a theoretical link to the loss in \eqref{eq:image_loss}. The solution we propose is an upper bound to $L$, which is constructed by relaxing the pixel weighting scheme. In general terms, we design a convex, compact space of weighting functions $\set W\subset \mathbb R^\set I$, subsuming the uniform weighting function, \ie $\{\frac{1}{\con n}\}^\set I\subset \set W$, and parametrize the loss function in \eqref{eq:image_loss} as 
\begin{equation}
    L_{w}(\hat y,y)=w\cdot \ell_{\hat yy}\,,
    \label{eq:parametrized_loss}
\end{equation}
with $w\in\set W$. Then, we define a new loss function $L_\set W\in\mathbb R^{\set Y\times\set Y}$, which targets the highest loss incurred with a weighting function in $\set W$, \ie
\begin{equation}
    L_{\set W}(\hat y,y)=\max\{L_w(\hat y,y)\,:\,w\in\set W\}\,.
    \label{eq:new_loss}
\end{equation}
Since the uniform weighting function belongs to $\set W$ and we maximize over $\set W$, it follows that $L_{\set W}$ upper bounds $L$, \ie $L_\set W(\hat y,y)\geq L(\hat y,y)$ for any $\hat y,y\in\set Y$. Consequently, we obtain an upper bound to \eqref{eq:training_objective} if we replace $L$ by $L_{\set W}$.

The title of our work, which ties the loss to max-pooling, is inspired by the observation that
the loss proposed in~\eqref{eq:new_loss} is the application of a generalized max-pooling operator acting on the pixel-losses. Indeed, we recover a conventional max-pooling operator as a special case if $\set W$ is the set of probability distributions over $\set I$. Similarly, the standard loss in~\eqref{eq:image_loss} can be regarded as the application of an average-pooling operator, which again can be boiled down to a special case of~\eqref{eq:new_loss} under a proper choice of $\set W$.

\myparagraph{The space $\set W$ of weighting functions.} 
The property that the loss max-pooling operator exhibits depends on the shape of $\set W$. %the space of weighting functions $\set W$.
Here, we restrict the focus to weighting functions with $p$-norm ($p\geq 1)$ and $\infty$-norm %. More precisely, we consider %the following set:
%a scenario where the $p$-norm and the $\infty$-norm (\aka max-norm) of any element in $\set W$ are 
upper bounded by $\gamma$ and $\tau$, respectively (see Fig.~\ref{fig:pnorm} for an example):
\begin{equation}
    \set W=\left\{w\in\mathbb R^\set I\,:\,\Vert w\Vert_p\leq \gamma,\,\Vert w\Vert_{\infty}\leq\tau\right\}\,.
    \label{eq:W}
\end{equation}
We fix the bound on the $p$-norm to $\gamma=\con n^{-1/q}$ with $q=\frac{p}{p-1}$, which corresponds to the $p$-norm of a uniform weighting function. 
Instead, $p$ and $\tau$ are left as hyper-parameters. Possible values of $\tau$ should be chosen in the range [$\con n^{-1},\gamma]$.
Indeed, lower values would prevent the uniform weighting function from belonging to $\set W$, while higher values would be equivalent to putting $\tau=\gamma$.
%If $\tau=\con n^{-1}$ or $p=\infty$ then $\set W$ contains only the uniform weighting function and, therefore, $L_\set W=L$.
%In all other cases $\set W$ includes the uniform weighting function.
%Parameter $\tau$ cannot be lower than $\con n^{-1}$ to prevent $\set W$ from being empty. Moreover, if $\tau=\con n^{-1}$ or $p=\infty$ then $\set W$ contains only the uniform weighting function and, therefore, $L_\set W=L$.
\begin{figure}[t]
    \centering
    \includegraphics[height=.39\columnwidth]{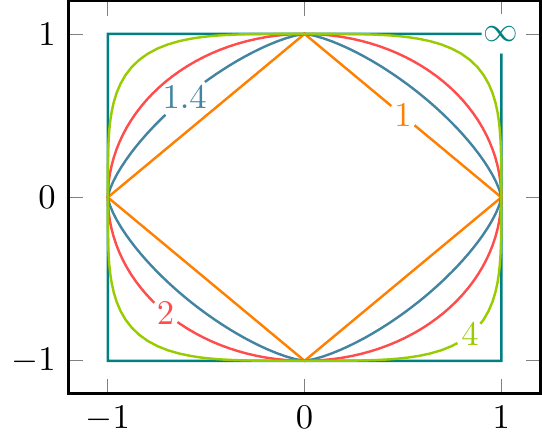}
    \includegraphics[height=.39\columnwidth]{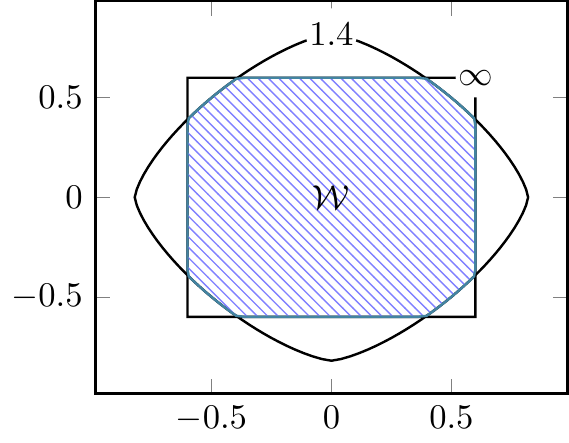}
    \caption{Left: Plot of $\Vert w\Vert_p=1$ in the $2$-dimensional case for $p\in\{1,1.4,2,4,\infty\}$. Right: Set $\set W$ when $n=2$, $p=1.4$ and $\tau=0.6$.}
    \label{fig:pnorm}
    \vspace{-5pt}
\end{figure}

Intuitively, the user can control the pixel selectivity degree of the pooling operation in \eqref{eq:new_loss} by changing $p$.
Indeed, the optimal weights will be in general concentrated around a single pixel as $p\to 1$ and be uniformly spread across pixels as $p\to\infty$.
On the other hand, $\tau$ allows to control, through the relation $m=\left(\frac{\gamma}{\tau}\right)^p$, the minimum number of pixels (namely $\lceil m\rceil$) that should be supported by the optimal weighting function.
%\mycomment{Add comment to Figure with optimal weightings given some arbitrary pixel losses under different parameterizations}
In Fig.~\ref{fig:weights} we show some examples, given synthetically-generated losses for $n=100$ pixels (sorted for better visualization). On the left, we fix $m=n/3$ (\ie at least $1/3$ of the pixels should be supported) and report the optimal weightings for different values of $p$. As we can see, the weights get more peaked on high losses as $p$ moves towards $1$, but the constraint on $m$ prevents selecting less than $\lceil m\rceil$ pixels. On the other hand, the weights tend to become uniform as $p$ approaches $\infty$. The plot on the right fixes $p=1.7$ and varies $m\in\{0,0.1n,0.2n,0.4n,0.8n,n\}$. We see that the weights tend to uniformly support a larger share of pixels as we increase $m$, yielding the uniform distribution when $m=n$.
\begin{figure*}[thb]
    \centering
    \includegraphics[width=\columnwidth]{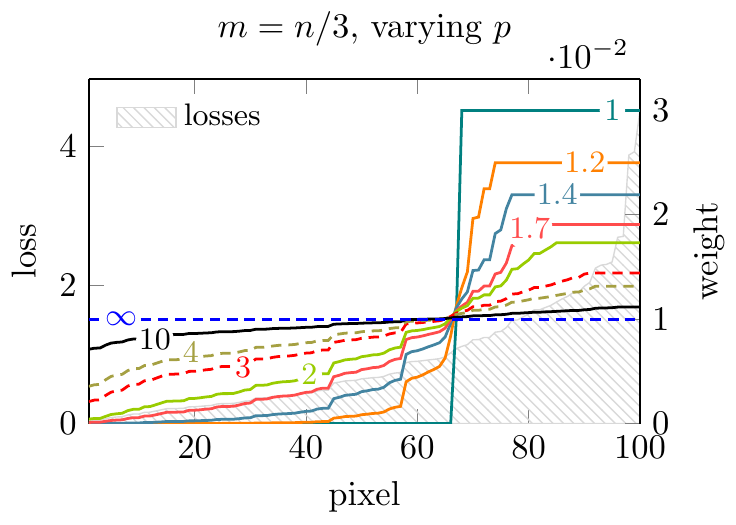}
    \includegraphics[width=\columnwidth]{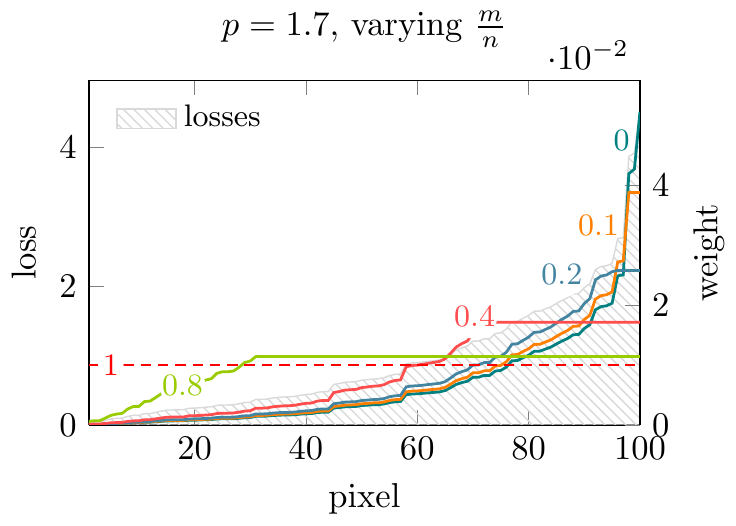}
	\vspace{-10pt}    
    \caption{Example of optimal weightings $w^*$ for $n=100$ pixels. Left: $m=n/3$ and varying values of $p\in\{1,1.2,1.4,1.7,2,3,4,10,\infty\}$. Right: $p=1.7$ and varying values of $\frac{m}{n}\in\{0,0.1,0.2,0.4,0.8,1\}$. Losses are synthetically generated and sorted for visualization purposes.}
    \label{fig:weights}
    \vspace{-10pt}
\end{figure*}

%In the next section we will give a closed-form characterization of \eqref{eq:new_loss}, when $\set W$ is defined as per \eqref{eq:W}.

\vspace{-10pt}\section{Computation of $L_\set W$}\label{sec:Computation}

The maximization problem in \eqref{eq:new_loss} is concave and has an explicit-form solution if $\set W$ is defined as in \eqref{eq:W}.
We provide the details by cases, considering the parametrization $(p,m)$ in place of $(p,\tau)$, because $m$ has a clear intuitive meaning as mentioned in the previous section. Valid parametrizations satisfy $p\geq 1$ and $1\leq m\leq n$.

%\subsection{Case $p=\infty$ or  $m=n$}
%This case yields the standard loss $L(\hat y,y)$ as a solution, since $\set W$ would only contain the uniform weighting function.
%Indeed, if $p=\infty$ then $q=1$ and $\gamma=n^{-1}$. It follows that no weight can exceed $\frac{1}{n}$

\subsection{Case $p>1$}
%For $p=1$, this solution takes the following form:
%\begin{equation}
%    L_{\set W}(\hat y,y)=\tau\left[\langle\ell_{\hat y y} \rangle_{\set J^*}+(m-|\set J^*|)\alpha^*\right]\,.
%    \label{}
%\end{equation}
%where $\set J^*\subset \set I$ is a set with $\lfloor m\rfloor $ pixels having the highest losses, and $\alpha^*$ is the $\lceil m \rceil$-th highest loss.
%
To address this case, we consider the following dual formulation of the maximization problem in~\eqref{eq:new_loss}:
\begin{equation}\label{eq:dual}
L_\set W(\hat y,y)=\min\left\{g(\lambda)\,:\lambda\succeq 0,\,\lambda\in\mathbb R^\set I\right\}\,,
\end{equation}
where $\lambda$ is the dual variable accounting for the constraint $w\preceq \tau$, which is equivalent to $\Vert w\Vert_\infty\leq\tau$, and
\[g(\lambda)=\tau \langle\lambda\rangle+\max\left\{w\cdot(\ell_{\hat y y}-\lambda):\Vert w\Vert_p\leq \gamma,\,w\in\mathbb R^{\set I}\right\}.\]
Moving from the primal to the dual formulation is legitimate because both formulations share the same optimal value. Indeed, the Slater's condition applies~\cite{BoyVan04} (\eg function $z\in\set I\mapsto 0$ is strictly feasible). % if we exclude the cases $\tau=n^{-1}$ and $p=\infty$ (as we did at the beginning of this section).

The maximization in $g(\lambda)$ is the definition of the \emph{dual norm} \cite[Appendix A.1.6]{BoyVan04} of the $p$-norm, which corresponds to the $q$-norm with $q=\frac{p}{p-1}$, evaluated in $\ell_{\hat y y}-\lambda$ and scaled by $\gamma$. Accordingly, we have that
\begin{equation}\label{eq:dual_obj}
    g(\lambda)=\tau\langle\lambda\rangle+\gamma\Vert \ell_{\hat y y}-\lambda\Vert_q\,.
\end{equation}
%This function is differentiable in all feasible points, but $\lambda=\ell_{\hat y y}$. 
%Minimizers of \eqref{eq:dual} corresponding to points where $g$ is differentiable can be characterized by means of the Karush-Kuhn-Tucker (KKT) conditions \cite{BoyVan04}, which provide the following necessary (and sufficient, by the convexity of $g$) condition that the optimal $\lambda$ should satisfy (see appendix for derivation):

We get a solution to~\eqref{eq:dual} by finding a point $\lambda^*$ that satisfies
\begin{equation}
    \lambda^*=\left(\ell_{\hat y y}-m^{-1/q} \Vert \ell_{\hat y y}-\lambda^*\Vert_q\right)_+
\label{eq:KKT}
\end{equation}
and maximizes $\Vert\ell_{\hat y y}-\lambda\Vert_q$ (see, Prop.~\ref{prop:KKT} in the appendix).
%\begin{remark}\label{rem:diff}
%Even though condition~\eqref{eq:KKT} is valid only on differentiable points, we can easily extend it to incorporate 
%also the point where $g$ is non-differentiable. It is sufficient to require that if multiple $\lambda$s exist that satisfy~\eqref{eq:KKT}, then 
%the one maximizing $\Vert\ell_{\hat y y}-\lambda\Vert_q$ should be chosen. 
%This procedure allows to correctly identify a valid solution, because 
%$\lambda=\ell_{\hat y y}$ always satisfies condition~\eqref{eq:KKT} and this will be the only solution that can be chosen in the case no solution exists on points where $g$ is differentiable. 
%Instead, if one or more such solutions exist, then $\lambda=\ell_{\hat y y}$, which is in general not a solution to~\eqref{eq:dual}, will be excluded, for it will not be a maximizer of $\Vert\ell_{\hat y y}-\lambda\Vert_q$.
%\end{remark}
However, computing such a solution from \eqref{eq:KKT} is not straightforward due to the recursive nature of the formula involving multiple variables (elements of $\lambda^*$).
We reduce it to the problem of finding the \emph{largest} root of the single-variable function
%, which is obtained from \eqref{eq:KKT} by setting $\alpha=m^{-1/q}\Vert \ell_{\hat y y}-\lambda\Vert_q$:
\begin{equation}\label{eq:KKT_alpha}
%    (m-|\set J_{\alpha}|)\alpha^q = \Vert \ell_{\hat y y}\Vert^q_{q,\overline{\set J}_\alpha}\,.
\eta(\alpha)=(m-|\set J_{\alpha}|)\alpha^q - \langle \ell^q_{\hat y y}\rangle_{\overline{\set J}_\alpha}\,,
\end{equation}
where
%Specifically, we consider the following expression for $\lambda$, which is equivalent to \eqref{eq:KKT}:
%\begin{equation}\label{eq:lambda}
%    \lambda(u)=
%    \begin{cases}
%        h(u)-\alpha&\text{if }u\in\set J_{\alpha}\\
%        0&\text{if }u\in \overline{\set J}_{\alpha}\,,
%    \end{cases}
%\end{equation}
%where 
$\set J_\alpha=\left\{ u\in\set I\,:\,\ell_{\hat y y}(u)> \alpha \right\}$ %corresponds to the subset of pixels where $\lambda$ is positive, 
and $\overline{\set J}_\alpha=\set I\setminus\set J_\alpha$ is its complement. %Note that a positive root of $\eta$ always exists, since \eqref{eq:dual} admits a finite solution.
%The reason why we aim for the largest root derives from Remark~\ref{rem:diff}. 
%In the appendix, we provide a proof of the equivalence between points satisfying condition~\eqref{eq:KKT} and roots of $\eta$.
%We show in Proposition~\ref{prop:alpha-sol} that the largest root of $\eta$ corresponds to a point $\lambda^*=|\ell_{\hat yy}-\alpha^*|_+$ satisfying~\eqref{eq:KKT} and maximizing $\Vert\ell_{\hat yy}-\ell\Vert_q$.
%The following lemma provides a formal link between roots of $\eta$ and solutions to~\eqref{eq:dual}:
%\begin{lemma}
%    Let $1\leq q<\infty$ and $1\leq m\leq n$. If $\lambda^*$ is a solution to~\eqref{eq:dual}, then $\alpha^*=m^{-1/q}\Vert\ell_{\hat y y}-\lambda^*\Vert_q$ is a root of $\eta$.
%    If $\alpha^*=\max\{\alpha\in\mathbb R:\eta(\alpha)=0\}$, then $\lambda^*=|\ell_{\hat y y}-\alpha^*|_+$ is a solution to~\eqref{eq:dual}.
%    \label{lem:alpha-sol}
%\end{lemma}
%Finally, we are able to explicitly compute the largest root of $\eta$ 
This characterization of solutions to the dual formulation~\eqref{eq:dual} in terms of roots of $\eta$ is proved correct in the appendix (Prop.~\ref{prop:alpha-sol}) and it is used to derive the theorem below,
which provides an explicit formula for $L_\set W(\hat y,y)$, the optimal weighting function $w^*$ of the maximization in~\eqref{eq:new_loss} and the optimal dual variable $\lambda^*$:

\begin{theorem}
    Let $1\leq q<\infty$, $1\leq m\leq n$ and $\alpha^*=\frac{\Vert\ell_{\hat y y}\Vert_{q,\overline{\set J}^*}}{(m-|\set J^*|)^{1/q}}$,
    where $\set J^*=\{u\in\set I\,:\,\eta(\ell_{\hat y y}(u))> 0\}$ and $\overline{\set J}^*=\set I\setminus\set J^*$. Then
\begin{equation}
    L_\set W(\hat y,y)=\tau\left[\langle \ell_{\hat y y}\rangle_{\set J^*}+ (m-|\set J^*|)\alpha^*\right]\,.
    \label{eq:final_obj}
\end{equation}
Moreover, $\lambda^*=|\ell_{\hat y y}-\alpha^*|_+$ is a minimizer of the dual formulation in~\eqref{eq:dual}, while
\[ w^*(u)=
    \begin{cases}
        \tau&\text{if }u\in\set J^*\\
        \tau\left(\displaystyle\frac{\ell_{\hat y y}(u)}{\alpha^*}\right)^{q-1}&\text{if }u\in\overline{\set J}^*\text{ and }\alpha^*>0\\
        0&\text{otherwise}
    \end{cases}
\] 
is a maximizer of the primal formulation in~\eqref{eq:new_loss}.
    \label{thm:main_theorem}
\end{theorem}

%In Subsection~\ref{ss:algo} we provide an algorithmic procedure to compute a solution as per Theorem~\ref{thm:fixed}.
%We can now support our previous claim about the existence of a solution to $\eqref{eq:dual}$ in a point where $g$ is differentiable, if $\lambda^*=\ell_{\hat y y}$ happens to be a solution as well. Indeed, if this is the case, by Lemma~\ref{lem:only_diff} there are at most $\lfloor m\rfloor$ positive element in $\ell_{\hat y y}$ and all belong to $\set J$, defined as per Theorem~\ref{thm:fixed}.

\subsection{Case $p=1$} \label{ss:special}
For this case, the solution takes the same form as in \eqref{eq:final_obj}, but $\set J^*$ becomes the subset of $\lfloor m \rfloor$ pixels with the highest losses, while $\alpha^*$ is the highest loss among the remaining pixels ($\alpha^*=0$ if $\set J^*=\set I$).
As for the optimal weighting function $w^*$, let $\set J^+=\{u\in\set I\,:\,\ell_{\hat y y}(u)=\alpha^*\}\setminus\set J^*$. Then for any probability distribution $\mu$ over $\set J^+$
\[
w^*(u)=
\begin{cases}
    \tau&\text{if }u\in\set J^*\\
    \tau(m-\lfloor m\rfloor)\mu(u)&\text{if }u\in\set J^+\\
    0&\text{otherwise,}
\end{cases}
\]
is an optimal solution for the primal (see, Thm.~\ref{thm:special} in the appendix).

\section{Algorithmic Details}\label{sec:alg}

%\subsection{Algorithm}\label{ss:algo}
%\mycomment{Fix description based on new code}
The key quantities to compute are $\set J^*$ and $\alpha^*$. Indeed, once those are available we can determine the loss $L_\set W(\hat y,y)$ and compute gradients with respect to the segmentation model's parameters (we show it later in this section).
We report in Algorithm~\ref{alg:J} the pseudo-code of the computation of $\set J^*$ and $\alpha^*$. 
%Note that we parametrize the algorithm with $(p,m)$ rather than $(p,\tau)$, because $m$ has a clear intuitive meaning as mentioned in Sec.~\ref{sec:method}.
%We exclude the trivial cases $m=n$ and $p=\infty$. Af for the other cases, 
We start sorting the losses (line 1). This yields a bijective function $\pi\in\set I^{\{1,\dots,n\}}$ satisfying $\ell_{\hat y y}(\pi_i)\leq\ell_{\hat y y}(\pi_j)$ if $i<j$ (we wrote $\pi_i$ for $\pi(i)$).
Case $p=1$ (line 13) is trivial, since we know that the last $\lfloor m\rfloor$ ranked pixels will form $\set J^*$, while $\alpha^*$ corresponds to the highest loss among the remaining pixels, or $0$ if no pixel is left (see, Subsection~\ref{ss:special}).
As for case $p>1$, we walk through the losses in ascending order and stop as soon as we find an index $i$
satisfying one of the following conditions: 
%$\eta(\ell_{\hat y y}(\pi_i))>0$, or we hit the end of the list ($i=n$).
a) $i=n$ and $\eta_n\leq 0$, or b) $\eta_i> 0$.
If the first condition is hit, then $\set J^*=\emptyset$ and, hence, $\alpha^*=\Vert\ell_{\hat y y}\Vert_q/m^{1/q}$. This is indeed what we obtain in line 11, where $i=n+1$ so that $\alpha^*=(a_n/c_n)^{1/q}$, where $c_n=m$ and $a_n=\sum_{j=1}^n\ell^q_{\hat y y}(\pi_j)$. Instead, if condition b is hit, then we have by Prop.~\ref{prop:algo} in the appendix that $\set J^*=\{\pi_j\,:\,i\leq j\leq n\}$. Consequently, $c_i=m-n+i=m-|\set J^*|$ and $a_i=\langle\ell^q_{\hat y y}\rangle_{\overline{\set J}^*}$ so that $\alpha^*=(a_i/c_i)^{1/q}$.

%\newlength{\oldsep}
%\setlength{\oldsep}{\textfloatsep}
%\setlength{\textfloatsep}{0pt}% Remove \textfloatsep
\begin{algorithm}[t]
	\begin{algorithmic}[1]
        \Require $m\in[1,n]$, $p\in[1,\infty]$, $n>0$ pixel losses $\ell_{\hat y y}$
        \State $\pi\leftarrow\text{sort}(\ell_{\hat yy})$
        \If{$p>1$}
        \State $q\leftarrow\frac{p}{p-1}$%,\quad\ell'_{\hat y y}\leftarrow\ell_{\hat y y}^q$
        \State $c_0\leftarrow m-n,\quad i\leftarrow 0,\quad a_0\leftarrow 0$ 
        \Repeat
%        \State $a'\leftarrow a,\quad a\leftarrow a+\ell'_{\hat y y}(\pi(i))$
        \State $i\leftarrow i+1,\quad c_i\leftarrow c_{i-1}+1$
        \State $a_i\leftarrow a_{i-1}+\ell^q_{\hat y y}(\pi_i)$
        \State $\eta_i\leftarrow c_i\,\ell^q_{\hat y y}(\pi_i)-a_i$
        \Until{$\eta_i>0$ \textbf{or} $i=n$}
        \If{$\eta_i\leq 0$}{} $i\leftarrow i+1$
        \EndIf
        \State $\alpha^*\leftarrow\left(\frac{a_{i-1}}{c_{i-1}}\right)^{1/q}$
%        \Else{} $\alpha^*\leftarrow \left(\frac{a_i}{c_i}\right)^{1/q},\quad i\leftarrow i+1$
%        \EndIf
        \Else
        \State $i\leftarrow n-\lfloor m \rfloor+1$
        \State $\alpha^*\leftarrow \ell_{\hat y y}(\pi_{i-1})\text{\bf~if }i>0\text{\bf~else }0$
        \EndIf
        \State \Return $\set J^*\leftarrow\{\pi_j\,:\, i\leq j\leq n\},\quad \alpha^*$
	\end{algorithmic}
\caption{Compute $\set J^*$, $\alpha^*$}
\label{alg:J}
\end{algorithm}
%\setlength{\textfloatsep}{\oldsep}

%we compute/compare all necessary quantities raised to the power of $q$ (indicated with the superscript $'$), because by doing so we can compute $\xi(\ell_{\hat y y}(u))^q$ (stored in variable $\xi^q$) incrementally, while traversing the loss-ordered pixels. 

\myparagraph{Gradient.}
In order to train the semantic segmentation model we need to compute the partial derivative $\frac{\partial L_{\set W}}{\partial \hat y}(\hat y,y)$. It exists 
\emph{almost} everywhere\footnote{Precisely, it exists in all $(\hat y,y)$ having an open neighborhood where $\set J^*$ does not change.} and is given by (see derivations in the appendix)%\mycomment{give derivations in suppmat}
\[
%    \frac{\partial L_{\set W}}{\partial \hat y(u)}(\hat y,y)=\tau\frac{\partial \ell_{\hat y y}}{\partial \hat y(u)}(u) w^*(u)\,. %c(u)\,,
    \frac{\partial L_{\set W}}{\partial \hat y}(\hat y,y)=\frac{\partial \ell_{\hat y y}}{\partial \hat y} w^*\,. %c(u)\,,
%    \tau\left[\left\langle \frac{\partial}{\partial\hat y}\ell_{\hat y y}\right\rangle_{\set J^*}+ (m-|\set J^*|)^\frac{1}{p}\Vert \ell_{\hat y y}\Vert_{q,\overline{\set J}^*}\right]
%\]
%where
%$
%    \begin{cases}
%        1&\text{if }u\in\set J^*\\
%        \left(\frac{\ell_{\hat y y}(u)}{\alpha^*}\right)^{q-1}&\text{otherwise.}
%    \end{cases}
\]
%$
Note that we will use the same function also where the partial derivative technically does not exist.\footnote{This is a common practice within the deep learning community (see, \eg how the derivative of ReLU, or max-pooling, are computed).}

\myparagraph{Implementation notes.} 
%\mycomment{Explain how to render it numerically more stable, by scaling the losses between 0 and 1, and inverse scaling the final loss.}
For values of $p$ close to 1, we have that $q$ becomes arbitrarily large and this might cause numerical issues in Algorithm~\ref{alg:J}. 
A simple trick to improve stability consists in normalizing the losses with a division by the maximum loss, \ie we consider $\frac{\ell_{\hat y y}}{\ell_{\hat y y}(\pi(n))}$ in place of $\ell_{\hat y y}$. This modification then requires multiplying $L_{\set W}(\hat y,y)$ by $\ell_{\hat y y}(\pi(n))$ to adjust the objective, while the optimal primal solution $w^*$ remains unaffected by the change.

\myparagraph{Complimentary sampling strategy.}
In addition to our main contribution described in the previous sections, we propose a complimentary idea on how to compile minibatches during training. We propose a mixed sampling approach taking both, uniform sampling from the training data's global distribution and the current performance of the model into account. As a surrogate for the latter, we keep track of the per-class Intersection over Union (IoU) scores on the training data and conduct inverse sampling, which will suggest to preferably pick from under-performing classes (that are often strongly correlated with minority classes). Blending this performance-based sampling idea with uniform sampling ensures to maintain stochastic behavior during training and therefore helps not to over-fit to particular classes.

%%%%%%%%%%%%%%%%%%%%%%%%%%%%%%%%%%%%%%%%%%%%%%%%%%%%%%%%%%%%%%%%%%%%
\section{Experiments}\label{sec:Experiments}
We have evaluated our novel loss max-pooling (\lmp) approach on the Cityscapes~\cite{Cordts2016} and the extended Pascal VOC~\cite{Everingham2010} semantic image segmentation datasets. In particular, we have performed an extensive parameter sweep on Cityscapes, assessing the performance development for different settings of our hyper-parameters $p$ and $m$ (see Equ.~\eqref{eq:W} and Fig.~\ref{fig:weights}). All reported numbers are Intersection-over-Union (IoU) (or \textit{Jaccard}) measures in $[\%]$, either averaged over all classes or provided on a per-class basis.

\subsection{Network architecture} 
For all experiments, we are using a network architecture similar to the one of DeepLabV2~\cite{Chen2016}, implemented within Caffe~\cite{Jia2014} using cuDNN for performance improvement and NCCL\footnote{\scriptsize\url{https://github.com/NVIDIA/nccl}} for multi-GPU support. In particular, we are using ResNet-101~\cite{He2015b} in a fully-convolutional way with atrous extensions~\cite{Holschneider1987,Yu2016} for the base layers before adding DeepLab's atrous spatial pyramid pooling (ASPP). Finally, we apply upscaling (via deconvolution layers with fixed, uniform weights and therefore performing bilinear upsampling) before using standard softmax loss for all baseline methods \baseOnly, \base and for the inverse median frequency weighting~\cite{Eigen2015} while we use our proposed loss max-pooling layer in \lmp. Both our approaches, \base and \lmp are using the complimentary sampling strategy for minibatch compilation as described in the previous section, while plain uniform sampling in \baseOnly leads to similar results as reported in~\cite{Chen2016}. We also report results of our new loss with plain uniform sampling (``Proposed loss only''). To save computation time and provide a conclusive parameter sensitivity study for our approach, we disabled both, multi-scale input to the networks and post processing via conditional random fields (CRF). We consider both of these features as complementary to our method and highly relevant for improving the overall performance in case time and hardware budgets permit to do so. However, our primary intention is to demonstrate the effectiveness of our \lmp, under comparable settings with other baselines like our \base. All our reported numbers and plots are obtained from fine-tuning the MS-COCO~\cite{LinMSCOCO2014} pre-trained CNN of~\cite{Chen2016}, which is available for download\footnote{\scriptsize\url{http://liangchiehchen.com/projects/DeepLabv2_resnet.html}}. In order to provide statistically more significant results, we provide mean and standard deviations obtained by averaging the results at certain steps over the last 30k training iterations. We only report results obtained from a single CNN as opposed to using an ensemble of CNNs, trained using the stochastic gradient descent (SGD) solver with polynomial decay of the learning rate ("poly" as described in~\cite{Chen2016}) setting both, decay rate and momentum to $0.9$. For data augmentation (Augm.), we use random scale perturbations in the range $0.5-1.5$ for patches cropped at positions given by the aforementioned sampling strategy, and horizontal flipping of images. 

\subsection{Cityscapes}
This recently-released dataset contains street-level images, taken at daytime from driving scenes in 50 major central European cities in Germany, France and Switzerland. Images are captured at high resolution (2.048 $\times$ 1.024) and are divided into training, validation and test sets holding 2.975, 500 and 1.525 images, respectively.  For training and validation data, densely annotated ground truth into 20 label categories (19 objects + ignore) is publicly available, where the 6 most frequent classes account for $\approx$90\% of the annotated pixel mass. Following previous works~\cite{Chen2016,Wu2016}, we report results obtained on the validation set. During training, we use minibatches comprising 2 image crops, each of size $550\times 550$. The initial learning rate is set to $2.5e^{-4}$, and we run a total number of $165k$ training iterations. %and is updated according to the "poly" learning rate update described in~\cite{Chen2016} using . Both, the "poly" decay rate and momentum used in SGD are fixed to $0.9$. 
\begin{table}[t]
%    \vspace{-10pt}
\centering
\resizebox{0.95\columnwidth}{!}{\footnotesize
\begin{tabular}{lccccc}
  \toprule
   $p$ & 150k & 160k & 165k & Mean & Std.Dev. \\
  \midrule
   1.0	& 74,35	& 74.64	& 74.64	& 74.54 & 0.17 \\
   1.1	& 74.34	& 74.61 &	74.60 &	74.52 &	0.15 \\
   1.2 &	74.42 &	74.60 &	74.77 &	\underline{74.60} &	0.18\\
   1.3 &	74.52 &	74.71 &	74.69 &	\textbf{74.64}&	0.10\\
   1.4 &	74.33 &	74.51 &	74.49 &	74.44&	0.10\\
   1.5 &	74.03 &	73.99 &	74.04 &	74.02&	0.03\\
   1.6 &	74.05 &	74.42 &	74.52 &	74.33&	0.25\\
   1.7 &	74.10 &	74.56 &	74.74 &	74.57&	0.17\\
   1.8 &	73.65 &	74.18 &	74.17 &	74.00&	0.30\\
   1.9 &	73.97 &	74.21 &	74.48 &	74.22&	0.26\\
   2.3 &    73.93 & 74.23 & 74.12 & 74.09 & 0.15\\
 \midrule
  \base & 73.12 & 73.16 & 73.10 & 73.13 & 0.03\\
  \bottomrule
\end{tabular}}
   \caption{Sensitivity analysis for $p$ parameter with $m$ fixed to $25\%$ of valid pixels per crop using \textit{efficient tiling} at test time. Numbers in $[\%]$ correspond to results on Cityscapes validation set after indicated training iterations (and averages with corresponding std.dev. thereof). Boldface and underlined values are in correspondence with best and second best results, respectively. Bottom-most row shows results from our baseline \base under the efficient tiling setting.}
   \label{tab:Sensitivity_p}
\vspace{-10pt}
\end{table}

In Tab.~\ref{tab:Sensitivity_p}, we provide a sensitivity analysis for hyper-parameter $p$, fixing $m$ to 25\% of non-ignore per-crop pixels. Due to the large resolution of images and considerable number of trainings to be run, we employ different tiling strategies during inference. 
Numbers in Tab.~\ref{tab:Sensitivity_p} are obtained by using our so-called \textit{efficient tiling} strategy, dividing the validation images into five non-overlapping, rectangular crops with full image height. With this setting, the best result was obtained for $p=1.3$, closely followed by $p=1.2$. It can be seen that increasing values for $p$ show a trend towards \base results, empirically confirming the theoretical underpinnings from Sect.~\ref{sec:method}. After fixing $p=1.3$, we conducted additional experiments with $m$ selected in a way to correspond to selecting at least $10\%$, $25\%$ or $50\%$ of non-ignore per-crop pixels, obtaining $74.09\%\pm 0.22,\, \mathbf{74.64}\pm 0.10$ and $73.44\pm 0.21$ on the validation data, respectively. Finally, we locked in on $p=1.3$ and $25\%$, running an \textit{optimized tiling} strategy on the validation set where we consider a 200 pixel overlap between tiles, allowing for improved context capturing. The final class label decisions for the first half of the overlap area are then exclusively taken by the left tile while the second half is provided by the right tile, respectively. The resulting scores are listed in Tab.~\ref{tab:CityscapesScores}, demonstrating improved results over \baseOnly, \base and related approaches like DeepLabV2~\cite{Chen2016} (even when using CRF) or~\cite{Wu2016} with deeper ResNet and online bootstrapping (BS). Also our loss alone, \ie without the complimentary sampling strategy, yields improved results over both \baseOnly and \base.
\begin{table}[t]
\centering
\resizebox{0.95\columnwidth}{!}{\footnotesize
    \begin{tabular}{lc}
  \toprule
   Method & mean IoU \\
  \midrule
   \cite{Chen2016} RN-101 \& Augm. \& ASPP & 71.0\\
   \cite{Chen2016} RN-101 \& Augm. \& ASPP \& CRF & 71.4\\
   \cite{Wu2016} FCRN-101 \& Augm. & 71.16\\
   \cite{Wu2016} FCRN-152 \& Augm. & 71.51\\
   \cite{Wu2016} FCRN-152 \& Online BS \& Augm. & 74.64\\
   \midrule
   \texttt{Our approaches - Resnet-101} & \\   
   \baseOnly Augm. \& ASPP & 72.55 $\pm 0.04$\\
   \base    Augm. \& ASPP & 73.63 $\pm 0.04$ \\
   \cite{Eigen2015} Inverse median freq. \& Augm. \& ASPP & 69.81 $\pm 0.08$ \\
   Proposed loss only \& Augm. \& ASPP &  74.17 $\pm 0.03$ \\
   \lmp    Augm. \& ASPP & \textbf{75.06} $\pm 0.09$\\   
  \bottomrule
\end{tabular}}
%\vspace{0.5em}
\caption{ResNet-based results (in $[\%]$) on validation set of Cityscapes dataset using \textit{optimized tiling}.}
   \label{tab:CityscapesScores}
   \vspace{-10pt}
\end{table}

To demonstrate the impact of our approach on under-represented classes, we provide a plot showing the per-class performance gain (\lmp - \base on $y$-axis in \%) vs.~the absolute number of pixels for a given object category ($x$-axis, log-scale) in Fig.~\ref{fig:Effectiveness}. Positive values on $y$ indicate improvements (18/19 classes) and class labels attached to $x$ indicate increasing object class pixel label volume for categories from left to right. \Eg, \textit{motorcycle} is most underrepresented while \textit{road} is most present. The plot confirms how \lmp naturally improves on underrepresented object classes \emph{without} accessing the underlying class statistics. 

Another experiment we have run compares \baseOnly to \lmp: In order to match the result of \lmp, one has to \eg improve the worst 7 categories by 5\% each or the worst 10 categories by 3\% each, which we find a convincing argument for \lmp. %  their IoU performances against \base, as a function of training iterations. We observe that \lmp (solid lines) consistently improves over the conventional log loss used in \base (dashed lines, same color) over a broad range of iterations. 
Additionally, we illustrate the qualitative evolution of the semantic segmentation for two training images in Fig.~\ref{fig:qualitative}. 
%We can see on the left the originl image (top) and the ground-truth segmentation (bottom). 
Odd rows show segmentations obtained when training with conventional log-loss in \base, while even rows show the ones obtained using our loss max-pooling \lmp at an increasing number of iterations. As we can see, \lmp starts improving on under-represented classes sooner than standard log-loss (see, \eg traffic light and its pole on the middle right in the first image, and the car driver in the second image).
\begin{figure}[th]
	\centering
		\includegraphics[width=.8\columnwidth]{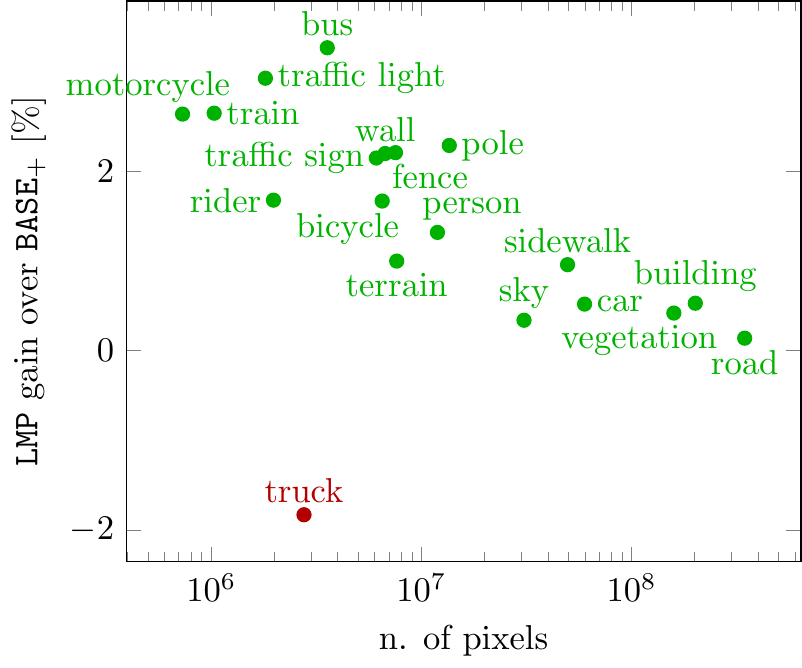}
        \vspace{-12pt}
        \caption{Improvement of \lmp over \base (18/19 classes) as a function of overall per-category pixel count on Cityscapes validation data.}
	\label{fig:Effectiveness}
    \vspace{-10pt}
\end{figure}
\begin{figure*}[h!]
    \newcolumntype{R}[2]{%
    >{\adjustbox{angle=#1,lap=\width-(#2)}\bgroup}%
    l%
    <{\egroup}%
}
    \setlength\tabcolsep{1.5pt}
    \centering
    \begin{tabular}{cccccccccc}
    \includegraphics[width=.1\textwidth]{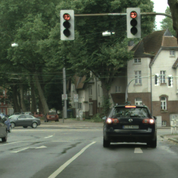}&&
    \includegraphics[width=.1\textwidth]{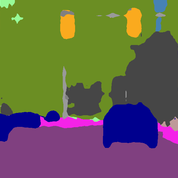}&
    \includegraphics[width=.1\textwidth]{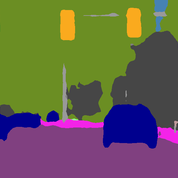}&
    \includegraphics[width=.1\textwidth]{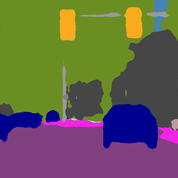}&
    \includegraphics[width=.1\textwidth]{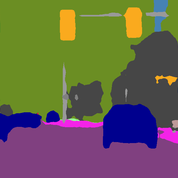}&
    \includegraphics[width=.1\textwidth]{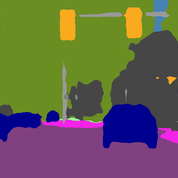}&
    \includegraphics[width=.1\textwidth]{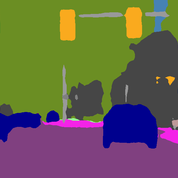}&
    \includegraphics[width=.1\textwidth]{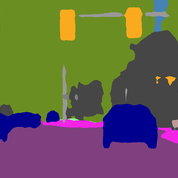}&
    \includegraphics[width=.1\textwidth]{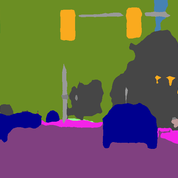}\\
    \includegraphics[width=.1\textwidth]{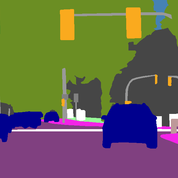}&&
    \includegraphics[width=.1\textwidth]{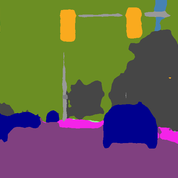}&
    \includegraphics[width=.1\textwidth]{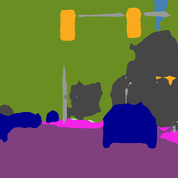}&
    \includegraphics[width=.1\textwidth]{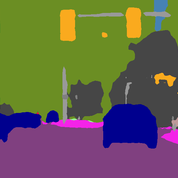}&
    \includegraphics[width=.1\textwidth]{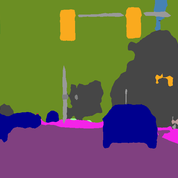}&
    \includegraphics[width=.1\textwidth]{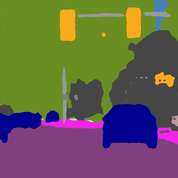}&
    \includegraphics[width=.1\textwidth]{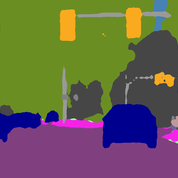}&
    \includegraphics[width=.1\textwidth]{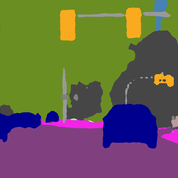}&
    \includegraphics[width=.1\textwidth]{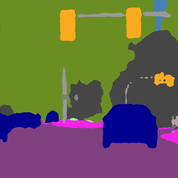}\\[3pt]
    \includegraphics[width=.1\textwidth]{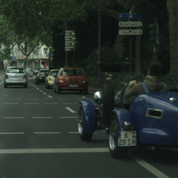}&&
    \includegraphics[width=.1\textwidth]{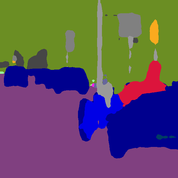}&
    \includegraphics[width=.1\textwidth]{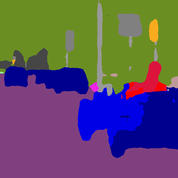}&
    \includegraphics[width=.1\textwidth]{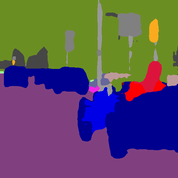}&
    \includegraphics[width=.1\textwidth]{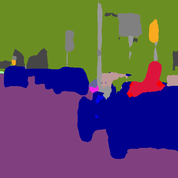}&
    \includegraphics[width=.1\textwidth]{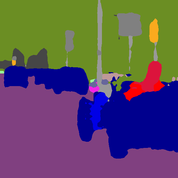}&
    \includegraphics[width=.1\textwidth]{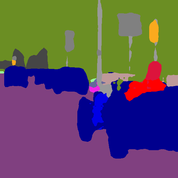}&
    \includegraphics[width=.1\textwidth]{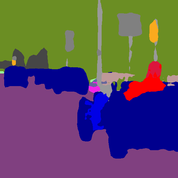}&
    \includegraphics[width=.1\textwidth]{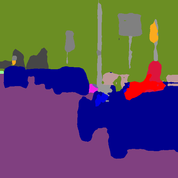}\\
    \includegraphics[width=.1\textwidth]{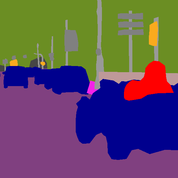}&&
    \includegraphics[width=.1\textwidth]{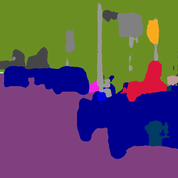}&
    \includegraphics[width=.1\textwidth]{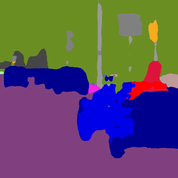}&
    \includegraphics[width=.1\textwidth]{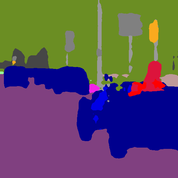}&
    \includegraphics[width=.1\textwidth]{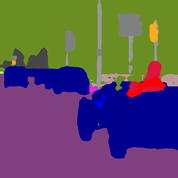}&
    \includegraphics[width=.1\textwidth]{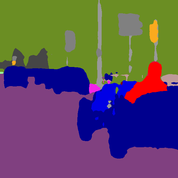}&
    \includegraphics[width=.1\textwidth]{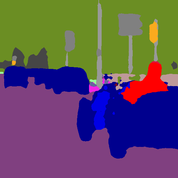}&
    \includegraphics[width=.1\textwidth]{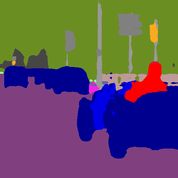}&
    \includegraphics[width=.1\textwidth]{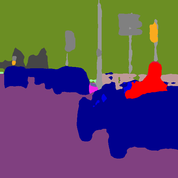}\\[3pt]

    \end{tabular}
    \vspace{-10pt}
    \caption{Evolution of semantic segmentation images during training. Left, we have pairs of original images (odd) and their ground-truth segmentations (even). The other images show semantic segmentations obtained by standard log-loss in \base (odd rows) and our loss max-pooling \lmp (even rows) after 20k, 40k, 60k, 80k, 100k, 120k, 140k, 165k training iterations.}
    \label{fig:qualitative}
\end{figure*}
%lr 2.5e-4 standard + adaptive
%inverse 1.125e-4 
Finally, we also report the individual per-class IoU scores in Tab.~\ref{tab:CityScapesClassIoU} for both, \base and \lmp, corresponding to the setting from Tab.~\ref{tab:CityscapesScores}. %It can be seen that we improve on all classes over our \base, except for category \textsc{Truck}, where we however also witness higher standard deviation. 
\begin{table*}
\vspace{-3pt}
    \setlength\tabcolsep{3pt}
\centering
{
    \subcaption{Cityscapes}\label{tab:CityScapesClassIoU}
\vspace{-4pt}

\resizebox{\textwidth}{!}{\begin{tabular}{lcccccccccccccccccccc}
  \toprule
   Method & Road& Sidewalk& Building& Wall& Fence& Pole& Traffic Light& Traffic Sign& Vegetation& Terrain& Sky& Person& Rider& Car& Truck& Bus& Train& Motorcycle& Bicycle& Mean \\
  \midrule
  \base Mean & 97.37 &  80.60 &  90.99 &  53.23 &  54.67 &  56.72 &  63.29 &  72.62 &  91.19 &  59.85 & 93.46 &  78.59 &  59.08 &  93.41 &  68.94 &  80.49 &  67.77 &  62.51 &  74.09 &  73.63\\
  \base Std.Dev. & 0.01 & 0.02 &  0.14 & 0.92 & 0.05 &  0.08 &  0.56 & 0.13 & 0.13 & 0.96 &  0.16 & 0.18 & 0.31& 0.12 & 0.26& 1.08 &  3.12 & 0.70  & 0.11 &  0.04\\
  \midrule
  \lmp Mean & 97.51 & 81.56 & 91.52 & 55.43 & 56.88 & 59.01 & 66.33 & 74.77 & 91.61 & 60.85 & 93.80 & 79.91 & 60.76 & 93.93 & 67.11 & 83.87 & 70.42 & 65.15 & 75.76 & 75.06 \\
  \lmp Std.Dev. & 0.05 & 0.32 & 0.06 & 0.80 & 0.68 & 0.31 & 0.27 & 0.21 & 0.07 & 0.27 & 0.09 & 0.08 & 0.29 & 0.01 & 0.77 & 0.44 & 0.53 & 0.44 & 0.14 & 0.09\\
  \bottomrule
\end{tabular}}
\vspace{.5em}
    \subcaption{Pascal VOC 2012}\label{tab:PascalClassIoU}
\vspace{-4pt}
\resizebox{\textwidth}{!}{\begin{tabular}{lcccccccccccccccccccccc}
  \toprule
   Method & Background &Aeroplane &Bicycle &Bird &Boat &Bottle &Bus &Car &Cat &Chair &Cow &Dining Table &Dog &Horse &Motorbike &Person &Potted Plant &Sheep &Sofa &Train &TV Monitor & Mean \\
  \midrule
  \base Mean & 92.69 &83.21 &78.46 &81.39 &67.95 &77.59 &92.14 &80.17 &86.99 &38.49 &80.86 &55.95 &81.03 &80.64 &79.28 &81.14 &61.74 &81.51 &47.76 &82.25 &72.68 &75.42 \\
  \base Std.Dev. & 0.00 &0.32 &0.04 &0.03 &0.14 &0.02 &0.12 &0.14 &0.08 &0.03 &0.06 &0.10 &0.02 &0.22 &0.06 &0.03 &0.35 &0.21 &0.24 &0.09 &0.24 &0.04 \\
  \midrule
 \lmp Mean & 92.84 &85.02 &79.62 &81.43 &69.99 &76.36 &92.38 &82.38 &89.43 &39.78 &82.70 &58.60 &82.85 &81.82 &80.17 &81.60 &61.22 &84.30 &45.44 &82.52 &71.70 &76.29 \\
 \lmp Std.Dev. & 0.02 &0.18 &0.03 &0.34 &0.26 &0.10 &0.17 &0.05 &0.10 &0.06 &0.17 &0.10 &0.05 &0.03 &0.07 &0.05 &0.14 &0.19 &0.04 &0.09 &0.10 &0.02 \\
  \bottomrule
\end{tabular}}
}
%\vspace{0.5em}
   \caption{Class-specific IoU scores on Cityscapes (with \textit{optimized tiling} during inference) and Pascal VOC 2012 validation datasets for our baseline (\base) and our proposed loss max-pooling (\lmp). All numbers in $[\%]$.}
   \vspace{-10pt}
\end{table*}

\subsection{Pascal VOC 2012}
We additionally assess the quality of our novel \lmp on the Pascal VOC 2012 segmentation benchmark dataset~\cite{Everingham2015}, comprising 20 object classes and a background class.
Images in this dataset are considerably smaller than the ones from the Cityscapes dataset so we increased the minibatch size to 4 (with crop sizes of $321\times 321$), using the (extended) training set with 10.582 images~\cite{Hariharan2011}. Testing was done on the validation set containing 1.449 images. We ran a total of 200k training iterations and fixed parameters $p=1.3$ and $m$ to account for $25\%$ of valid pixels per crop for our \lmp. During inference, images are evaluated at full scale, \ie~no special tiling mechanism is needed. We again report the mean IoU scores in Tab.~\ref{tab:PascalScores} (this time averaged after training iterations 180k, 190k and 200k), and list results from comparable state-of-the-art approaches~\cite{Wu2016,Chen2016} next to ours.
\begin{table}[th]
\centering
\resizebox{\columnwidth}{!}{\footnotesize
    \begin{tabular}{lc}
  \toprule
   Method & mean IoU \\
  \midrule
   \cite{Chen2016} RN-101 Base w/o COCO & 68.72\\
   \cite{Chen2016} RN-101 \& MSC \& Augm. \& ASPP  & 76.35\\   
   \cite{Chen2016} RN-101 \& MSC \& Augm. \& ASPP \& CRF & 77.69\\
   \cite{Wu2016} FCRN-101 \& Augm. & 73.41\\
   \cite{Wu2016} FCRN-152 \& Augm. & 73.32\\
   \cite{Wu2016} FCRN-101 \& Online BS \& Augm. & 74.80\\   
   \cite{Wu2016} FCRN-152 \& Online BS \& Augm. & 74.72\\
   \midrule
   \texttt{Our approaches - Resnet-101} & \\
   \baseOnly Augm. \& ASPP & 75.74 $\pm 0.05$\\
   \base    Augm. \& ASPP & 75.42 $\pm 0.04$ \\
   \cite{Eigen2015} Inverse median freq. \& Augm. \& ASPP & 74.93 $\pm 0.03$ \\
   Proposed loss only \& Augm. \& ASPP &  76.01 $\pm 0.01$ \\
   \lmp    Augm. \& ASPP & 76.29 $\pm 0.02$ \\
  \bottomrule
\end{tabular}}
%\vspace{0.5em}
   \caption{ResNet based results on Pascal VOC 2012 segmentation validation data. All numbers in $[\%]$.}
   \label{tab:PascalScores}
   \vspace{-15pt}
\end{table}
We can again obtain a considerable relative improvement over \base as well as comparable baselines from~\cite{Wu2016,Chen2016}. Our approach compares slightly worse ($-1.4\%$) with DeepLabV2's strongest variant, which however additionally uses multi-scale inputs (MSC) and refinements from a CRF (contributing $2.55\%$ and $1.34\%$ according to~\cite{Chen2016}, respectively) but only come with increased computational costs. Additionally, and as mentioned above, we consider both of these techniques as complementary to our contributions and plan to integrate them in future works. We finally notice that also for this dataset our new loss alone without the complimentary sampling strategy yields consistent improvements over \baseOnly and \base.

In Tab.~\ref{tab:PascalClassIoU}, we give side-by-side comparisons of per class IoU scores for \base and \lmp. Again, the majority of categories benefits from our approach, confirming its efficacy.
%lr 2.5e-4 standard + adaptive
% 6.25e-5 inverse <- had to be tweaked

%%%%%%%%%%%%%%%%%%%%%%%%%%%%%%%%%%%%%%%%%%%%%%%%%%%%%%%%%%%%%%%%%%%%
\section{Conclusions}\label{sec:Conclusion}
In this work we have introduced a novel approach to tackle imbalances in training data distributions, which do not occur only when we have under-represented classes (inter-class imbalance), but might occur also within the same class (intra-class imbalance).
%The solution we proposed consists in the definition of a new loss function that performs a generalized max-pooling of pixel-specific losses. 
We proposed a new loss function that performs a generalized max-pooling of pixel-specific losses. 
Our loss upper bounds the traditional one, which gives equal weight to each pixel contribution, and implicitly introduces an adaptive weighting scheme that biases the learner towards under-performing image parts. The space of weighting functions involved in the maximization can be shaped to enforce some desired properties. In this paper we focused on a particular family of weighting functions, enabling us to control the pixel selectivity and the extent of the supported pixels. We have derived explicit formulas for the outcome of the pooling operation under this family of pixel weighting functions, thus enabling the computation of gradients for training deep neural networks. We have experimentally validated the effectiveness of our new loss function and showed consistently improved results on standard benchmark datasets for semantic segmentation.

{\small
\myparagraph{Acknowledgements.} We gratefully acknowledge financial support from project \textit{DIGIMAP}, funded under grant \#860375 by the Austrian Research Promotion Agency (FFG).
}

\appendix

\onecolumn

\section{Proofs and Auxiliary Results}\label{sec:proofs}

\begin{proof}[Proof of Thm.~\ref{thm:main_theorem}]
    Let $\alpha^\star=\max\{\alpha\in\mathbb R\,:\,\eta(\alpha)=0\}$ and let $\set J^\star=\{u\in\set I:\eta(\ell_{\hat y y}(u))>
\alpha^\star\}$.
    %By Lemma~\ref{lem:div-by-zero} we have that $|\set J^\star|<m$ and
    %by Lemma~\ref{lem:alpha} we have that $\lambda^\star=|\ell_{\hat y y}-\alpha^\star|_+$ is a solution to~\eqref{eq:dual}.
    We start proving that $\set J^\star=\set J^*$. 
    If $u\in\set J^*$, then by definition of $\set J^*$ we have $\eta(\ell_{\hat y y}(u))>0$, which 
    %By the contrapositive of Lemma~\ref{lem:monotonicity}(a) we have 
    implies $\ell_{\hat y y}(u)>\alpha^\star$ by Proposition~\ref{prop:contrapp} (take $\alpha_1=\ell_{\hat y y}(u)$ and $\alpha_2=\alpha^\star$).
    Consequently, $u\in\set J^\star$ and, therefore, $\set J^*\subseteq\set J^\star$.
    If $u\in\set J^\star$, then by definition of $\set J^\star$ we have $\eta(\ell_{\hat y y}(u))>\alpha^\star$, which implies $\eta(\ell_{\hat y y}(u))>0$ since $\alpha^\star\geq 0$ hold by Proposition~\ref{prop:div-by-zero}. 
    %We can then invoke Proposition~\ref{prop:contrapp} to prove that $\ell_{\hat y y}(u)>\alpha^\star$ (take $\alpha_1=\ell_{\hat y y}(u)$ and $\alpha_2=\alpha^\star$). This and the fact that $|\set J^\star|<m$ (by Proposition~\ref{prop:div-by-zero}) imply $\eta(\ell_{\hat y y}(u))>0$ by Proposition~\ref{prop:monotonicity} (take $\alpha_1=\alpha^\star$ and $\alpha_2=\ell_{\hat y y}(u)$). 
    Therefore, $u\in\set J^*$. %However, $\eta(\ell_{\hat y y}(u))=0$ is excluded, for it would contradict the definition of $\alpha^\star$. 
    So, $\set J^\star\subseteq\set J^*$. We have thus proved that $\set J^\star=\set J^*$.
    
    We obtain $\alpha^*$, as given in the theorem, by solving equation $\eta(\alpha^\star)=0$ for variable $\alpha^\star$, after having replaced $\set J^\star$ with the equivalent $\set J^*$. The equation admits a unique solution, because $|\set J^*|=|\set J^\star|<m$ by Proposition~\ref{prop:div-by-zero}. Accordingly, $\alpha^*=\alpha^\star$. % and, consequently, $\lambda^*=\lambda^\star$ is a solution to~\eqref{eq:dual}.
    %Accordingly, $\alpha^*$ coincides with $\alpha^\star$, which in turn is the largest root of $\eta$.
    %Finally, by Lemma~\ref{lem:div-by-zero} we have that $|\set J^\star|=|\set J|<m$ always holds and, therefore, $\alpha^*$ always exists and is finite.  
    Then, by Proposition~\ref{prop:alpha-sol} we have that $\lambda^*=|\ell_{\hat y y}-\alpha^*|_+$ is a solution to~\eqref{eq:dual}, % and $\alpha^*=m^{-1/q}\Vert\ell_{\hat y y}-\lambda^*\Vert_q$, 
    from which we derive %We subs into $\eqref{eq:g}$ to obtain $L_{\set W}(\hat y,y)$:
    \begin{multline*}
        L_{\set W}(\hat y, y)=g(\lambda^*)
        =\tau\langle\lambda^*\rangle+\gamma \Vert\ell_{\hat y y}-\lambda^*\Vert_q
        =\tau\langle\lambda^*\rangle+\gamma (\langle\ell_{\hat y y}^q\rangle_{\overline{\set J}^*}+\alpha^{*q}|\set J^*|)^{1/q}\\
        =\tau\langle\lambda^*\rangle+\gamma (m\alpha^{*q}-\underbrace{\eta(\alpha^*)}_{=0})^{1/q}
        =\tau(\langle\ell_{\hat y y}\rangle_{\set J^*}-|\set J^*|\alpha^*)+\underbrace{\gamma m^{1/q}}_{=\tau m}\alpha^*
        =\tau\left[\langle\ell_{\hat y y}\rangle_{\set J^*}+(m-|\set J^*|)\alpha^*)\right]\,.
    \end{multline*}
%        \begin{equation}
%        \end{aligned}

    As for $w^*$, we have that $L_\set W(\hat y,y)\geq w^*\cdot\ell_{\hat y, y}$ holds in general.
    %If equality holds, then $w^*$ is an optimal solution to the primal objective. This is trivially the case 
    Now, if $\alpha^*=0$, then $L_{\set W}(\hat y,y)\geq w^*\cdot\ell_{\hat y y}=\tau\langle\ell_{\hat y y}\rangle_{\set J^*}=L_{\set W}(\hat y, y)$.
    %which implies that $w^*$ is optimal for the primal objective.     
    If $\alpha^*>0$, then
    \[
%        \begin{aligned}
            L_{\set W}(\hat y,y)\geq w^*\cdot\ell_{\hat y y}=\tau\langle\ell_{\hat y y}\rangle_{\set J^*}+\frac{\tau}{(\alpha^*)^{q-1}}\left\langle\ell^q_{\hat y y} \right\rangle_{\overline{\set J}^*}
            =\tau\left[\langle\ell_{\hat y y}\rangle_{\set J^*}+(m-|\set J^*|)\alpha^*\right]=L_{\set W}(\hat y,y)\,,
%        \end{aligned}
    \]
    where the last equality follows from the observation that $(m-|\set J^*|)\alpha^{*q}=\langle\ell_{\hat y y}^{q}\rangle_{\overline{\set J}^*}$, by definition of $\alpha^*$. Hence, $w^*$ is an optimal solution to the maximization in~\eqref{eq:new_loss}.
\end{proof}

\begin{proposition}
    Let $1\leq q<\infty$ and $1\leq m\leq n$. If $\lambda^*$ is a solution to~\eqref{eq:dual}, then $\alpha^*=m^{-1/q}\Vert\ell_{\hat y y}-\lambda^*\Vert_q$ is a root of $\eta$.
    If $\alpha^*=\max\{\alpha\in\mathbb R:\eta(\alpha)=0\}$, then $\lambda^*=|\ell_{\hat y y}-\alpha^*|_+$ is a solution to~\eqref{eq:dual}.
    \label{prop:alpha-sol}
\end{proposition}
\begin{proof}
Let $\lambda^*$ be a solution to~\eqref{eq:dual}. It follows from Propositions~\ref{prop:KKT} and~\ref{prop:alpha} that $\alpha^*=m^{-1/q}\Vert\ell_{\hat y y}-\lambda^*\Vert_q$ is a root of $\eta$.

Let $\alpha^*=\max\{\alpha\in\mathbb R:\eta(\alpha)=0\}$ and $\lambda^*=|\ell_{\hat y y}-\alpha^*|_+$. 
Then
\[
    m^{-1/q}\Vert\ell_{\hat y y}-\lambda^*\Vert_q=m^{-1/q}(\langle\ell_{\hat y y}^q\rangle_{\overline{\set J}^*}+\alpha^{*q}|\set J^*|)^{1/q}=m^{-1/q}(m\alpha^{*q}-\underbrace{\eta(\alpha^*)}_{=0})^{1/q}=\alpha^*.
\]
%By Proposition~\ref{prop:alpha} we have that $\lambda^*$ satisfies~\eqref{eq:KKT} and from the previous equation we have $\alpha^*=m^{-1/q}\Vert\ell_{\hat y y}-\lambda^*\Vert_q$.
Now, let $\lambda^\star\in\argmax\{\Vert\ell_{\hat y y}-\lambda\Vert_q\,:\,\lambda\text{ satisfies \eqref{eq:KKT}}\}$ and let $\alpha^\star=m^{-1/q}\Vert\ell_{\hat y y}-\lambda^\star\Vert_q$. By Proposition~\ref{prop:alpha} we have that $\lambda^*$ satisfies~\eqref{eq:KKT}, hence 
\[
    \alpha^*=m^{-1/q}\Vert\ell_{\hat y y}-\lambda^*\Vert_q\leq m^{-1/q}\Vert\ell_{\hat y y}-\lambda^\star\Vert_q=\alpha^\star\,.
\]
By Proposition~\ref{prop:alpha}, $\alpha^\star$ is a root of $\eta$. However,
this implies  $\alpha^*=\alpha^\star$ by definition of $\alpha^*$ and, therefore, $\lambda^*=\lambda^\star$. Finally,
by Proposition~\ref{prop:KKT} we conclude that $\lambda^*$ is a solution to~\eqref{eq:dual}.
\end{proof}

\begin{theorem}
    Let $p=1$ and $1\leq m \leq n$. Then
    \begin{equation}
        L_{\set W}(\hat y,y)=\tau[\langle\ell_{\hat y y}\rangle_{\set J^*} + (m-|\set J^*|)\alpha^*]\,,
        \label{eq:special_sol}
    \end{equation}
    where $\set J^*\in\argmax\{\langle\ell_{\hat y y}\rangle_{\set J}\,:\,\set J\subseteq\set I,|\set J|=\lfloor m\rfloor\}$,
    \ie $\set J^*$ contains the pixels with the $\lfloor m\rfloor$ highest losses, 
    and $\alpha^*=\Vert\ell_{\hat y y}\Vert_{\infty,\overline{\set J}^*}$, \ie it corresponds to the highest loss in $\overline{\set J}^*$  or zero if $\overline{\set J}^*$ is empty. 
    Moreover, 
    \[
w^*(u)=
\begin{cases}
    \tau&\text{if }u\in\set J^*\\
    \tau(m-\lfloor m\rfloor)\mu(u)&\text{if }u\in\set J^+\\
    0&\text{otherwise}
\end{cases}
    \]
    is an optimal solution for the maximization in~\eqref{eq:new_loss}, for any probability distribution $\mu$ defined over $\set J^+=\{u\in\set I:\ell_{\hat y y}(u)=\alpha^*\}\setminus\set J^*$.
    \label{thm:special}
\end{theorem}
\begin{proof}
    Assume $w^\star$ to be the maximizer in~\eqref{eq:new_loss}, \ie $L_{\set W}(\hat y,y)=w^\star\cdot\ell_{\hat y y}$. 
    Then it has to be nonnegative and should satisfy $\Vert w^\star\Vert_1=\gamma$. 
Otherwise, we could construct $w^\dag=\gamma |w^\star|/\Vert w^\star\Vert_1$, which would satisfy $w^\dag\cdot \ell_{\hat y y}>w^\star\cdot\ell_{\hat y y}$ contradicting $L_{\set W}(\hat y,y)=w^\star\cdot\ell_{\hat y y}$.

    Now, from~\eqref{eq:new_loss} and the definition of $w^*$ we can derive
    \[
        L_{\set W}(\hat y,y)\geq w^*\cdot\ell_{\hat y y}=\tau[\langle\ell_{\hat y y}\rangle_{\set J^*}+(m-\underbrace{\lfloor m\rfloor}_{=|\set J^*|})\underbrace{\langle\mu\ell_{\hat y y}\rangle_{\set J^+}}_{=\alpha^*}].
\]
Assume by contradiction that strict inequality holds, or in other terms that $w^\star\cdot\ell_{\hat y y}>w^*\cdot\ell_{\hat y y}$.
%Then $L_{\set W}(\hat y,y)=w^\star\cdot\ell_{\hat y y}>w^*\cdot\ell_{\hat y y}$ holds for some $w^\star\in\set W\setminus\{ w^*\}$ being nonnegative and satisfying $\Vert w^\star\Vert_1=\gamma$.
Let $\set A^+=\{u\in\set I:w^\star(u)>w^*(u)\}$ and $\set A^-=\{u\in\set I:w^\star(u)<w^*(u)\}$.
Note that $\set A^+\neq\emptyset$ and $w^\star\neq w^*$, otherwise $w^\star\cdot\ell_{\hat y y}\leq w^*\cdot\ell_{\hat y y}$ holds yielding a contradiction, and $\set A^+\subseteq\overline{\set J}^*$, because $w^\star$ is upper bounded by $\tau$ and $w^*(u)=\tau$ for $u\in\set J^*$.
It follows by definition of $\alpha^*$ that $\ell_{\hat y y}(u)\leq\alpha^*$ for any $u\in\set A^+$.
Additionally, $\set A^-\neq\emptyset$, otherwise $\Vert w^\star\Vert_1 > \Vert w^*\Vert_1=\gamma$ holds contradicting $w^\star\in\set W$,
and necessarily $\set A^-\subseteq\set J^*\cup\set J^+$ because $w^\star$ is lower bounded by $0$ and $w^*(u)=0$ for $u\notin\set A^-$. 
It follows by definition of $\set J^*$ and $\set J^+$ that $\ell_{\hat y y}(u)\geq\alpha^*$ for any $v\in\set A^-$.
But then
\[
    (w^\star-w^*)\cdot\ell_{\hat y y}=\underbrace{\langle(w^\star-w^*)\ell_{\hat y y}\rangle_{\set A^+}}_{\leq \alpha^*\langle w^\star-w^*\rangle_{\set A^+}} + \underbrace{\langle(w^\star-w^*)\ell_{\hat y y}\rangle_{\set A^-}}_{\leq\alpha^*\langle w^\star-w^*\rangle_{\set A^-}}\leq \alpha^*\langle w^\star -w^*\rangle =\alpha^*(\underbrace{\langle w^\star\rangle}_{=\gamma}-\underbrace{\langle w^*\rangle}_{=\gamma})=0\,,
\]
%\end{align*}
yielding a contradiction. Hence, $L_{\set W}(\hat y,y)=w^*\cdot\ell_{\hat y y}$, \ie $w^*$ is an optimal solution for the maximization in~\eqref{eq:new_loss}, and~\eqref{eq:special_sol} holds.
\end{proof}

\begin{proposition}
    Let $1\leq q<\infty$ and $1\leq m\leq n$. If $\lambda$ satisfies~\eqref{eq:KKT}, then $\alpha=m^{-1/q}\Vert\ell_{\hat y y}-\lambda\Vert_q$ is a root of $\eta$. % (defined as per~\eqref{eq:KKT_alpha}). 
    If $\alpha$ is a root of $\eta$, then $\lambda=|\ell_{\hat y y}-\alpha|_+$ satisfies \eqref{eq:KKT}.
    \label{prop:alpha}
\end{proposition}
\begin{proof}
    Let $\alpha=m^{-1/q}\Vert\ell_{\hat y y}-\lambda\Vert_q$. If $\lambda$ satisfies~\eqref{eq:KKT}, then $\lambda=|\ell_{\hat y y}-\alpha|_+$. By substituting it back into $\alpha$ we obtain:
    \[
        \begin{aligned}
            \alpha^q&=m^{-1}\Vert\ell_{\hat y y}-\lambda\Vert^q_q\\
            m\alpha^q&=\left[|\set J_\alpha| \alpha^q + \langle \ell^q_{\hat y y}\rangle_{\overline{\set J}_\alpha} \right]\\
            0&=(m-|\set J_{\alpha}|)\alpha^q-\langle\ell^q_{\hat y y}\rangle_{\overline{\set J}_\alpha}=\eta(\alpha)\,.
        \end{aligned}
    \]
    Hence, $\alpha$ is a root of $\eta$.

    Let $\alpha$ be a root of $\eta$ and let $\lambda=|\ell_{\hat y y}-\alpha|_+$. 
    By following the previous relation bottom-up, we obtain $\alpha=m^{-1/q}\Vert\ell_{\hat y y}-\lambda\Vert_q$, and by substituting it back into $\lambda$ we obtain~\eqref{eq:KKT}.
\end{proof}

\begin{proposition}
    Let $1\leq q<\infty$ and $1\leq m\leq n$. If $\lambda^*$ is a solution to~\eqref{eq:dual}, then it satisfies~\eqref{eq:KKT}.
    If $\lambda^*\in\argmax\{\Vert\ell_{\hat yy}-\lambda\Vert_q\,:\,\lambda \text{ satisfies \eqref{eq:KKT}}\}$, then it is a solution to~\eqref{eq:dual}.
    \label{prop:KKT}
\end{proposition}
\begin{proof}
    Let $\lambda^*$ be a solution to~\eqref{eq:dual}.
    If $\lambda^*=\ell_{\hat y y}$, then \eqref{eq:KKT} is trivially satisfied. Otherwise, it is satisfied by Proposition~\ref{prop:KKT-diff}.

    Let $\lambda^*\in\argmax\{\Vert\ell_{\hat yy}-\lambda\Vert_q\,:\,\lambda\text{ satisfies \eqref{eq:KKT}}\}$. If $\lambda^*\neq\ell_{\hat yy}$, then $\lambda^*$ is a solution to~\eqref{eq:dual} by Proposition~\ref{prop:KKT-diff}. If $\lambda^*=\ell_{\hat yy}$, then it is the only point satisfying~\eqref{eq:KKT}. It follows from Proposition~\ref{prop:KKT-diff} that no solution to~\eqref{eq:dual} exists where $g$ is differentiable.
    However, at least one solution has to exist because the minimization problem in~\eqref{eq:dual} admits a finite solution. So, it has to be a point where $g$ is non-differentiable, but the only one is $\lambda^*=\ell_{\hat yy}$. Therefore, $\lambda^*$ is a solution to~\eqref{eq:dual}.
\end{proof}

\begin{proposition}\label{prop:KKT-diff}
    Let $1\leq q<\infty$, $1\leq m\leq n$ and $\lambda^*\neq\ell_{\hat y y}$. Then $\lambda^*$ is a solution to~\eqref{eq:dual} if and only if it satisfies~\eqref{eq:KKT}.
\end{proposition}
\begin{proof}
    ($\Rightarrow$) If $\lambda^*$ is a solution to~\eqref{eq:dual} and $\lambda^*\neq\ell_{\hat y y}$, then $\lambda^*$ is a point where $g$ is differentiable and the Karush-Kuhn-Tucker~(KKT) necessary conditions~\cite{BoyVan04} for optimality are satisfied. Specifically, there exists $\nu\succeq 0$ satisfying
    \[
\tau-\gamma\left(\frac{\ell_{\hat y y}-\lambda^*}{\Vert\ell_{\hat y y}-\lambda^*\Vert_q}\right)^{q-1}-\nu=0\,,\qquad\nu\cdot\lambda^*=0\,.
    \]
    The complementarity constraint and the nonnegativity of $\nu$ imply that $\nu(u)=0$ if $\lambda^*(u)>0$.
    By using this fact, we can derive after simple algebraic manipulations the following equivalent relation, which holds for all $u\in\set I$:
    \[
        \lambda^*(u)=
        \begin{cases}
            \ell_{\hat y y}(u)-m^{-1/q}\Vert\ell_{\hat y y}-\lambda^*\Vert_q&\text{if }\lambda^*(u)>0\\
            0&\text{otherwise,}
        \end{cases}
    \]
%    with $m=(\gamma/\tau)^p$. 
    and this corresponds to 
    \[
        \lambda^*=|\ell_{\hat y y}-m^{-1/q}\Vert\ell_{\hat y y}-\lambda^*\Vert_q|_+\,.
    \]
    
    ($\Leftarrow$) By following the derivation above in reversed order, we have that if $\lambda^*\neq\ell_{\hat y y}$ satisfies~\eqref{eq:KKT} then the KKT conditions for optimality are satisfied. Since $g$ is convex, those conditions are also sufficient~\cite{BoyVan04} and, therefore, $\lambda^*$ is a solution to~\eqref{eq:dual}.
\end{proof}

\begin{proposition}
    Let $1\leq q<\infty$ and $1\leq m\leq n$. If $\alpha^\star=\max\{\alpha\in\mathbb R:\eta(\alpha)=0\}$, then $|\set J^\star|<m$ and $\alpha^\star\geq 0$, where $\set J^\star=\{u\in\set I:\eta(\ell_{\hat y y}(u))>\alpha^\star\}$.
    \label{prop:div-by-zero}
\end{proposition}
\begin{proof}
    Let $\lambda^*$ be a solution to~\eqref{eq:dual}. By Proposition~\ref{prop:alpha-sol}
    %It satisfies~\eqref{eq:KKT} and by Proposition~\ref{prop:alpha} 
    we have that $\alpha^*=m^{-1/q}\Vert\ell_{\hat y y}-\lambda^*\Vert_q$ is a root of $\eta$. 
    If $\lambda^*=\ell_{\hat y y}$ we have $\alpha^*=0$ and by Proposition~\ref{prop:only_diff} we have that $\ell_{\hat y y}$ has at most $\lfloor m\rfloor$ positive elements. It follows that $|\set J^*|\leq m$. 
    If $\lambda^*\neq\ell_{\hat y y}$, we have $\alpha^*>0$. This and $\eta(\alpha^*)=0$ imply $|\set J^*|\leq m$.
    Accordingly, $|\set J^*|\leq m$ always holds and since $\alpha^\star\geq\alpha^*\geq 0$, we have $|\set J^\star|\leq|\set J^*|\leq m$ and $\alpha^\star\geq 0$.

    Finally, assume by contradiction that $|\set J^\star|= m$ holds. Then $\eta(\alpha^\star)=0$ implies $\langle\ell^q_{\hat y y}\rangle_{\overline{\set J}^\star}=0$.  Take $\underline\alpha=\min\{\ell_{\hat y y}(u)\,:\,u\in\set J^\star\}$, then $\set J_{\underline\alpha}\subset\set J^\star$, where $\set J_{\underline\alpha}=\{u\in\set I:\eta(\ell_{\hat y y}(u))>\underline\alpha\}$. Let $\Delta=\set J^\star\setminus\set J_{\underline\alpha}$, which is necessarily non-empty and contains only pixels with loss $\underline\alpha$. Then
\[%    \begin{multline*}
        \eta(\underline\alpha)=(m-|\set J_{\underline\alpha}|)\underline\alpha^q - \langle\ell^q_{\hat y y}\rangle_{\overline{\set J}_{\underline\alpha}}
        =(\underbrace{m-|\set J^\star|}_{=0}+|\Delta|)\underline\alpha^q-\underbrace{\langle\ell^q_{\hat y y}\rangle_{\overline{\set J}^\star}}_{=0}-\langle\ell_{\hat y y}^q\rangle_\Delta
        =|\Delta|\underline\alpha^q-\langle\ell^q_{\hat y y}\rangle_\Delta=0\,,
    \]%    \end{multline*}
    where the last equality follows from the fact that $\ell_{\hat y y}(u)=\underline\alpha$ for all $u\in\Delta$. 
    However, $\eta(\underline\alpha)=0$ implies $\alpha^\star\geq\underline\alpha$ by definition of $\alpha^\star$, which yields a contradiction because $\underline\alpha>\alpha^\star$ follows from the definition of $\underline\alpha$. Hence, $|\set J^\star|<m$.
\end{proof}

\begin{proposition}
    Let $p>1$ and $1\leq m\leq n$. If $\lambda^*=\ell_{\hat y y}$ is a minimizer of~\eqref{eq:dual}, then $\ell_{\hat y y}$ has at most $\lfloor m\rfloor$ positive elements. 
    \label{prop:only_diff}
\end{proposition}
\begin{proof}
    If $p=\infty$ then $m=n$ and the result is trivially true. 
    Otherwise ($1<p<\infty$), assume by contradiction that there exist at least $\lfloor m\rfloor +1$ positive elements in $\ell_{\hat y y}$, say with indices in $\set J$. Then, the dual objective yields  $g(\lambda^*)=\tau\langle\ell_{\hat y y}\rangle_\set J$. However, this value cannot be attained by the primal formulation because there exist at most $\lfloor m \rfloor$ elements in $\set J$ with weight $\tau$, while the remaining element would have a weight not exceeding $\tau (m-\lfloor m\rfloor)^{1/p}<\tau$ (due to the constraint $\Vert w\Vert_p\leq \gamma$). This implies $L_{\set W}(\hat y,y)<g(\lambda^*)$, which contradicts strong duality, implied by the satisfaction of the Slater's condition~\cite{BoyVan04}.
\end{proof}

\begin{proposition}
    Let $1\leq q<\infty$ and $1\leq m\leq n$. If $0\leq\alpha_1<\alpha_2$, $|\set J_{1}|<m$ and $\eta(\alpha_1)\geq 0$, then $\eta(\alpha_2)>0$, where $J_1=\{u\in\set I:\ell_{\hat y y}(u)>\alpha_1\}$.
    \label{prop:monotonicity}
\end{proposition}
\begin{proof}
    Let $J_2=\{u\in\set I:\ell_{\hat y y}(u)>\alpha_2\}$.
    The hypothesis $\alpha_1<\alpha_2$ implies that $\set J_{2}\subseteq\set J_{1}$. Hence, we can write
\[%    \begin{align*}
        \eta(\alpha_2) = (m-|\set J_1|)\alpha_2^q + |\Delta|\alpha_2^q - \langle\ell_{\hat y y}^q\rangle_{\overline{\set J}_1} - \langle\ell_{\hat y y}^q\rangle_{\Delta}
        =\underbrace{\eta(\alpha_1)}_{\geq 0}+\underbrace{(m-|\set J_1|)\delta}_{a}+\underbrace{|\Delta|\alpha_2^q- \langle\ell_{\hat y y}^q\rangle_{\Delta}}_b\,,
    \]%    \end{align*}
    where $\delta=\alpha_2^q-\alpha_1^q$ and $\Delta=\set J_1\setminus\set J_2$.
    %Now, the nonnegativity of $\eta(\alpha_1)$ and positivity of $\alpha_1$ imply $|\set J_1|\leq m$. Hence 
    Expression $a$ is positive, because $\delta>0$ and $|\set J_1|<m$.
    Moreover, since $\Delta$ has no element in $\set J_2$, we have by definition of $\set J_2$ that $\ell_{\hat y y}(u)\leq\alpha_2$ for all $u\in\Delta$ and, therefore, $b$ is nonnegative. It follows that $\eta(\alpha_2)>0$.
\end{proof}

\begin{proposition}
    Let $1\leq q<\infty$ and $1\leq m\leq n$. If $\alpha_1\geq 0$, $\eta(\alpha_1)>0$ and $\eta(\alpha_2)\leq 0$, then $\alpha_2<\alpha_1$. \label{prop:contrapp}
\end{proposition}
\begin{proof}
    By Proposition~\ref{prop:aux} we have $|\set J_1|<m$, where $J_1=\{u\in\set I:\ell_{\hat y y}(u)>\alpha_1\}$.
    The result then follows from the contrapositive of Proposition~\ref{prop:monotonicity}.
\end{proof}

\begin{proposition}
    Let $1\leq q<\infty$ and $1\leq m\leq n$. If $\alpha\geq 0$ and $\eta(\alpha)>0$, then $|\set J_\alpha|<m$, where $J_\alpha=\{u\in\set I:\ell_{\hat y y}(u)>\alpha\}$.\label{prop:aux}
\end{proposition}
\begin{proof}
    If $\eta(\alpha)>0$, then $(m-|\set J_\alpha|)\alpha^q$ must be positive. But this is the case only if $|\set J_\alpha|<m$, 
    since $\alpha$ is nonnegative.
\end{proof}

\begin{proposition}
    \label{prop:algo}
    Let $1\leq q<\infty$, let $\pi$ be a bijective function $\pi\in\set I^{\{1,\dots,n\}}$ satisfying $\ell_{\hat y y}(\pi_i)\leq\ell_{\hat y y}(\pi_j)$ if $i<j$, and let 
    \[
        \tau=\argmin\quad\{i\in\{1,\dots,n\}:\eta_i>0\}\cup\{n+1\}\,,
    \]
    where $\eta_i=(m-n+i)\ell^q_{\hat y y}(\pi_i)-\sum_{j=1}^i\ell_{\hat y y}^q(\pi_j)$. Then
    \[
        \{\pi_j\,:\, \tau\leq j \leq n\}=\{u\in\set I:\eta(\ell_{\hat y y}(u))>0\}\,.
    \]
\end{proposition}
\begin{proof}
%    We start proving that $\eta(\ell_{\hat y y}(\pi_i))=\eta_i$ for all $1\leq i\leq n$.
    Take $i\in\{1,\dots,n\}$, let $\set J_i=\{u\in\set I:\eta(\ell_{\hat y y}(u))>\ell_{\hat y y}(\pi_i)\}$ and let $u=\max\{j\in\{1,\dots,n\}:\ell_{\hat y y}(\pi_j)=\ell_{\hat y y}(\pi_i)\}$. Then
    \begin{equation}
    \begin{aligned}
    \eta(\ell_{\hat y y}(\pi_i))=(m-|\set J_i|)\ell^q_{\hat y y}(\pi_i)-\langle\ell^q_{\hat y y}\rangle_{\overline{\set J}_i}&=
    (m-n+u)\ell^q_{\hat y y}(\pi_i)-\sum_{j=1}^u\ell^q_{\hat y y}(\pi_j)\\
    &\stackrel{(*)}{=}(m-n+i)\ell^q_{\hat y y}(\pi_i)-\sum_{j=1}^i\ell^q_{\hat y y}(\pi_j)=\eta_i\,,
    \end{aligned}
        \label{eq:equiv}
    \end{equation}
    holds, where in $(*)$ we used the fact that $\ell^q_{\hat y y}(\pi_j)$ is constant for $i\leq j\leq u$. It follows that $\eta(\ell_{\hat y y}(\pi_i))\leq 0$ for all $1\leq i<\tau$. Now, if $\tau=n+1$ then the theorem trivially holds, for sets $\{\pi_j\,:\, \tau\leq j \leq n\}$ and $\{u\in\set I:\eta(\ell_{\hat y y}(u))>0\}$ are empty.
    If $\tau\leq n$ then $\eta_\tau>0$ by definition of $\tau$ and, hence, $\eta(\ell_{\hat y y}(\pi_\tau))>0$ by~\eqref{eq:equiv}. Consequently, by Proposition~\ref{prop:aux} and Proposition~\ref{prop:monotonicity} we have that $\eta(\ell_{\hat y y}(\pi_j))>0$ for all $\tau\leq j\leq n$. 
\end{proof}

\section{Derivation of Gradient}\label{sec:gradient}

As discussed in the main paper, the gradient $\frac{\partial L_\set W}{\partial\hat y}(\hat y,y)$ exists almost everywhere. 
For the points where it exists, the gradient takes the form:
\[
    \frac{\partial L_\set W}{\partial\hat y}(\hat y,y)=w^*\cdot\frac{\partial\ell_{\hat y y}}{\partial\hat y}+\frac{\partial w^*}{\partial\hat y}\cdot \ell_{\hat y y}.
\]
In general we consider gradients in directions that leave $\set J^*$ and $\set J^+$ unchanged.
Under this assumption we have that $\frac{\partial w^*}{\partial\hat y}\cdot \ell_{\hat y y}=0$ so that $\frac{\partial L_\set W}{\partial\hat y}(\hat y,y)=w^*\cdot\frac{\partial\ell_{\hat y y}}{\partial\hat y}$.

Indeed, $\frac{\partial w^*}{\partial\hat y}=0$ holds for the case $p=1$.
For the case $p>1$, note that the optimal $w^*$ always satisfies $\Vert w^*\Vert_p=\gamma$, which implies that
$\frac{\partial w^*}{\partial \hat y}\cdot w^{*(p-1)}=0$ has to be satisfied, indeed
\[
    \begin{aligned}
    \frac{\partial}{\partial \hat y}\Vert w^*\Vert_p&=\frac{\partial}{\partial \hat y}\gamma\\
    \frac{\partial w^*}{\partial\hat y}\cdot\left(\frac{w^*}{\Vert w^*\Vert_p}\right)^{p-1}&=0\\
    \frac{\partial w^*}{\partial\hat y}\cdot\left(\frac{w^*}{\gamma}\right)^{p-1}&=0\\
    \frac{\partial w^*}{\partial\hat y}\cdot w^{*(p-1)}&=0\,.
    \end{aligned}
\]
Now, if we consider $w^*$ as per Theorem~\ref{thm:main_theorem} in the main paper, we have that $\frac{\partial w^*}{\partial\hat y}(u)=0$ for $u\in\set J^*$. Hence $\frac{\partial w^*}{\partial \hat y}\cdot w^{*(p-1)}=0$ implies 
\[
    \begin{aligned}
        \frac{\partial w^*}{\partial \hat y}\cdot w^{*(p-1)}
        =\sum_{u\in\set J^*}\underbrace{\frac{\partial w^*}{\partial\hat y}(u)}_{=0}w^{*(p-1)}(u)+ \sum_{u\in\overline{\set J}^*}\frac{\partial w^*}{\partial\hat y}(u)\tau^{p-1}\frac{\ell_{\hat y y}(u)}{\alpha^*}&=0\\
         \sum_{u\in\overline{\set J}^*}\frac{\partial w^*}{\partial\hat y}(u)\ell_{\hat y y}(u)&=0\\
         \sum_{u\in\set J^*}\underbrace{\frac{\partial w^*}{\partial\hat y}(u)}_{=0}\ell_{\hat y y}(u)+ \sum_{u\in\overline{\set J}^*}\frac{\partial w^*}{\partial\hat y}(u)\ell_{\hat y y}(u)&=0\\
        \frac{\partial w^*}{\partial \hat y}\cdot\ell_{\hat y y}&= 0\,.
\end{aligned}
\]

\twocolumn

{\small
\bibliographystyle{ieee}

\begin{thebibliography}{10}\itemsep=-1pt

\bibitem{Ahmed2015}
F.~Ahmed, D.~Tarlow, and D.~Batra.
\newblock Optimizing expected intersection-over-union with
  candidate-constrained crfs.
\newblock In {\em (ICCV)}, pages 1850--1858, 2015.

\bibitem{Badrinarayanan2015}
V.~Badrinarayanan, A.~Kendall, and R.~Cipolla.
\newblock Segnet: A deep convolutional encoder-decoder architecture for image
  segmentation.
\newblock {\em arXiv preprint arXiv:1511.00561}, 2015.

\bibitem{Bansal2016}
A.~Bansal, X.~Chen, B.~Russell, A.~Gupta, and D.~Ramanan.
\newblock Pixelnet: Towards a general pixel-level architecture.
\newblock {\em CoRR}, abs/1609.06694, 2016.

\bibitem{Blaschko2008}
M.~Blaschko and C.~Lampert.
\newblock Learning to localize objects with structured output regression.
\newblock In {\em (ECCV)}, 2008.

\bibitem{BoyVan04}
S.~P. Boyd and L.~Vandenberghe.
\newblock {\em Convex Optimization}.
\newblock Cambridge University Press, 2004.

\bibitem{Bunkhumpornpat2009}
C.~Bunkhumpornpat, K.~Sinapiromsaran, and C.~Lursinsap.
\newblock {\em Safe-Level-{SMOTE}: Safe-Level-Synthetic Minority Over-Sampling
  TEchnique for Handling the Class Imbalanced Problem}, pages 475--482.
\newblock Springer Berlin Heidelberg, Berlin, Heidelberg, 2009.

\bibitem{Caesar2015}
H.~Caesar, J.~R.~R. Uijlings, and V.~Ferrari.
\newblock Joint calibration for semantic segmentation.
\newblock In {\em (BMVC)}, 2015.

\bibitem{Chawla2002}
N.~V. Chawla, K.~W. Bowyer, L.~O. Hall, and W.~P. Kegelmeyer.
\newblock {SMOTE:} synthetic minority over-sampling technique.
\newblock {\em J. Artificial Intell. Res. (JAIR)}, 16:321--357, 2002.

\bibitem{Chen2016}
L.~Chen, G.~Papandreou, I.~Kokkinos, K.~Murphy, and A.~L. Yuille.
\newblock Deeplab: Semantic image segmentation with deep convolutional nets,
  atrous convolution, and fully connected {CRF}s.
\newblock {\em CoRR}, abs/1606.00915, 2016.

\bibitem{Cordts2016}
M.~Cordts, M.~Omran, S.~Ramos, T.~Rehfeld, M.~Enzweiler, R.~Benenson,
  U.~Franke, S.~Roth, and B.~Schiele.
\newblock The cityscapes dataset for semantic urban scene understanding.
\newblock In {\em (CVPR)}, 2016.

\bibitem{Deng2009}
J.~Deng, W.~Dong, R.~Socher, L.-J. Li, K.~Li, and L.~Fei-Fei.
\newblock {ImageNet: A Large-Scale Hierarchical Image Database}.
\newblock In {\em (CVPR)}, 2009.

\bibitem{Eigen2015}
D.~Eigen and R.~Fergus.
\newblock Predicting depth, surface normals and semantic labels with a common
  multi-scale convolutional architecture.
\newblock In {\em (ICCV)}, pages 2650--2658, 2015.

\bibitem{Everingham2015}
M.~Everingham, S.~M.~A. Eslami, L.~Van~Gool, C.~K.~I. Williams, J.~Winn, and
  A.~Zisserman.
\newblock The pascal visual object classes challenge: A retrospective.
\newblock {\em International Journal of Computer Vision}, 111(1):98--136, 2015.

\bibitem{Everingham2010}
M.~Everingham, L.~Van~Gool, C.~K.~I. Williams, J.~Winn, and A.~Zisserman.
\newblock The pascal visual object classes {(VOC)} challenge.
\newblock {\em (IJCV)}, 88(2):303--338, 2010.

\bibitem{FeiFei2006}
L.~Fei-fei, R.~Fergus, and P.~Perona.
\newblock One-shot learning of object categories.
\newblock {\em (PAMI)}, 28:2006, 2006.

\bibitem{Ghiasi2016}
G.~Ghiasi and C.~C. Fowlkes.
\newblock Laplacian reconstruction and refinement for semantic segmentation.
\newblock {\em CoRR}, abs/1605.02264, 2016.

\bibitem{Griffin2007}
G.~Griffin, A.~Holub, and P.~Perona.
\newblock Caltech-256 object category dataset.
\newblock Technical Report 7694, California Institute of Technology, 2007.

\bibitem{Gulcehre2013}
{\c{C}}.~G{\"{u}}l{\c{c}}ehre, K.~Cho, R.~Pascanu, and Y.~Bengio.
\newblock Learned-norm pooling for deep neural networks.
\newblock {\em CoRR}, abs/1311.1780, 2013.

\bibitem{Han2005}
H.~Han, W.-Y. Wang, and B.-H. Mao.
\newblock {\em Borderline-SMOTE: A New Over-Sampling Method in Imbalanced Data
  Sets Learning}, pages 878--887.
\newblock Springer Berlin Heidelberg, 2005.

\bibitem{Hariharan2011}
B.~Hariharan, P.~Arbeláez, L.~Bourdev, S.~Maji, and J.~Malik.
\newblock Semantic contours from inverse detectors.
\newblock In {\em (ICCV)}, pages 991--998, Nov 2011.

\bibitem{He2015b}
K.~He, X.~Zhang, S.~Ren, and J.~Sun.
\newblock Deep residual learning for image recognition.
\newblock {\em CoRR}, abs/1512.03385, 2015.

\bibitem{Holschneider1987}
M.~Holschneider, R.~Kronland-Martinet, J.~Morlet, and P.~Tchamitchian.
\newblock A real-time algorithm for signal analysis with the help of the
  wavelet transform.
\newblock In J.-M. Combes, A.~Grossmann, and P.~Tchamitchian, editors, {\em
  Wavelets: Time-Frequency Methods and Phase Space}, pages 286--297. Springer
  Berlin Heidelberg, 1987.

\bibitem{Huang2016}
C.~Huang, Y.~Li, C.~C. Loy, and X.~Tang.
\newblock Learning deep representation for imbalanced classification.
\newblock In {\em (CVPR)}, 2016.

\bibitem{Jeatrakul2010}
P.~Jeatrakul, K.~W. Wong, and C.~C. Fung.
\newblock Classification of imbalanced data by combining the complementary
  neural network and {SMOTE} algorithm.
\newblock In {\em Intern. Conf. on Neural Information Processing}, pages
  152--159. Springer, 2010.

\bibitem{Jia2014}
Y.~Jia, E.~Shelhamer, J.~Donahue, S.~Karayev, J.~Long, R.~Girshick,
  S.~Guadarrama, and T.~Darrell.
\newblock Caffe: Convolutional architecture for fast feature embedding.
\newblock {\em arXiv preprint arXiv:1408.5093}, 2014.

\bibitem{Khan2015}
S.~H. Khan, M.~Bennamoun, F.~A. Sohel, and R.~Togneri.
\newblock Cost sensitive learning of deep feature representations from
  imbalanced data.
\newblock {\em CoRR}, abs/1508.03422, 2015.

\bibitem{Khoshgoftaar2007}
T.~M. Khoshgoftaar, M.~Golawala, and J.~Van~Hulse.
\newblock An empirical study of learning from imbalanced data using random
  forest.
\newblock In {\em 19th IEEE International Conference on Tools with Artificial
  Intelligence (ICTAI 2007)}, volume~2, pages 310--317. IEEE, 2007.

\bibitem{Kontschieder2013}
P.~Kontschieder, P.~Kohli, J.~Shotton, and A.~Criminisi.
\newblock {GeoF}: {G}eodesic forests for learning coupled predictors.
\newblock In {\em (CVPR)}, pages 65--72, 2013.

\bibitem{CIFAR10}
A.~Krizhevsky, V.~Nair, and G.~Hinton.
\newblock Cifar-10 (canadian institute for advanced research).

\bibitem{Lecun98b}
Y.~Lecun, L.~Bottou, Y.~Bengio, and P.~Haffner.
\newblock Gradient-based learning applied to document recognition.
\newblock In {\em Proceedings of the IEEE}, pages 2278--2324, 1998.

\bibitem{Lin2015}
G.~Lin, C.~Shen, I.~Reid, et~al.
\newblock Efficient piecewise training of deep structured models for semantic
  segmentation.
\newblock {\em arXiv preprint arXiv:1504.01013}, 2015.

\bibitem{LinMSCOCO2014}
T.~Lin, M.~Maire, S.~J. Belongie, L.~D. Bourdev, R.~B. Girshick, J.~Hays,
  P.~Perona, D.~Ramanan, P.~Doll{\'{a}}r, and C.~L. Zitnick.
\newblock Microsoft {COCO:} common objects in context.
\newblock {\em CoRR}, abs/1405.0312, 2014.

\bibitem{Liu2015}
W.~Liu, A.~Rabinovich, and A.~C. Berg.
\newblock Parsenet: Looking wider to see better.
\newblock {\em CoRR}, abs/1506.04579, 2015.

\bibitem{Long2015}
J.~Long, E.~Shelhamer, and T.~Darrell.
\newblock Fully convolutional networks for semantic segmentation.
\newblock In {\em (CVPR)}, pages 3431--3440, 2015.

\bibitem{Mostajabi2015}
M.~Mostajabi, P.~Yadollahpour, and G.~Shakhnarovich.
\newblock Feedforward semantic segmentation with zoom-out features.
\newblock In {\em (CVPR)}, June 2015.

\bibitem{Nowozin2014}
S.~Nowozin.
\newblock Optimal decisions from probabilistic models: the
  intersection-over-union case.
\newblock In {\em (CVPR)}, June 2014.

\bibitem{Ranjbar2010}
M.~Ranjbar, G.~Mori, and Y.~Wang.
\newblock Optimizing complex loss functions in structured prediction.
\newblock In {\em (ECCV)}, 2010.

\bibitem{Raskutti2004}
B.~Raskutti and A.~Kowalczyk.
\newblock Extreme re-balancing for {SVM}s: a case study.
\newblock {\em ACM Sigkdd Explorations Newsletter}, 6(1):60--69, 2004.

\bibitem{RotNeuKon17cvpr}
S.~Rota~Bul\`o, G.~Neuhold, and P.~Kontschieder.
\newblock Loss max-pooling for semantic image segmentation.
\newblock In {\em (CVPR)}, 2017.

\bibitem{Shen2015}
W.~Shen, X.~Wang, Y.~Wang, X.~Bai, and Z.~Zhang.
\newblock Deepcontour: A deep convolutional feature learned by positive-sharing
  loss for contour detection.
\newblock In {\em (CVPR)}, June 2015.

\bibitem{Shrivastava2016}
A.~Shrivastava, A.~Gupta, and R.~B. Girshick.
\newblock Training region-based object detectors with online hard example
  mining.
\newblock In {\em (CVPR)}, 2016.

\bibitem{Tang2009}
Y.~Tang, Y.-Q. Zhang, N.~V. Chawla, and S.~Krasser.
\newblock {SVM}s modeling for highly imbalanced classification.
\newblock {\em IEEE Transactions on Systems, Man, and Cybernetics},
  39(1):281--288, 2009.

\bibitem{Wu2016}
Z.~Wu, C.~Shen, and A.~van~den Hengel.
\newblock High-performance semantic segmentation using very deep fully
  convolutional networks.
\newblock {\em CoRR}, abs/1604.04339, 2016.

\bibitem{Xu2014}
J.~Xu, A.~G. Schwing, and R.~Urtasun.
\newblock Tell me what you see and i will show you where it is.
\newblock In {\em (CVPR)}, 2014.

\bibitem{Xu2015}
J.~Xu, A.~G. Schwing, and R.~Urtasun.
\newblock Learning to segment under various forms of weak supervision.
\newblock In {\em (CVPR)}, 2015.

\bibitem{Yu2016}
F.~Yu and V.~Koltun.
\newblock Multi-scale context aggregation by dilated convolutions.
\newblock {\em Int. Conf. on Learning Representations (ICLR)}, 2016.

\bibitem{Zheng2015}
S.~Zheng, S.~Jayasumana, B.~Romera-Paredes, V.~Vineet, Z.~Su, D.~Du, C.~Huang,
  and P.~H. Torr.
\newblock Conditional random fields as recurrent neural networks.
\newblock In {\em (ICCV)}, pages 1529--1537, 2015.

\end{thebibliography}

}

\end{document}